\documentclass{article}

\usepackage{microtype}
\usepackage{graphicx}
\usepackage{subfigure}
\usepackage{booktabs} 
\usepackage{multirow}
\usepackage{makecell}
\usepackage{capt-of}
\usepackage{algorithm}
\usepackage{algorithmic}
\usepackage{hyperref}
\usepackage[dvipsnames]{xcolor}
\usepackage{makecell}
\newcommand{\theHalgorithm}{\arabic{algorithm}}

\usepackage[accepted]{icml2024}

\usepackage{amsmath}
\usepackage{amssymb}
\usepackage{mathtools}
\usepackage{amsthm}

\usepackage{pifont}
\newcommand{\cmark}{\ding{51}}%
\newcommand{\xmark}{\ding{55}}%

\usepackage[capitalize,noabbrev]{cleveref}

\theoremstyle{plain}
\newtheorem{theorem}{Theorem}[section]
\newtheorem{proposition}[theorem]{Proposition}

\theoremstyle{definition}

\newtheorem{assumption}[theorem]{Assumption}
\theoremstyle{remark}
\newtheorem{remark}[theorem]{Remark}

\usepackage[textsize=tiny]{todonotes}
\usepackage{xspace}
\def\wrt{w.r.t.\@\xspace}
\usepackage[normalem]{ulem}

\newcommand{\weburl}{\url{brbiclab.epfl.ch/projects/turtle}\xspace}

\begin{document}

\newcommand{\methodold}{{HUME}\xspace}
\newcommand{\methodnew}{{TURTLE}\xspace}
\newcommand{\xhdr}[1]{\noindent{{\bf #1.}}}

\twocolumn[
\icmltitle{Let Go of Your Labels with Unsupervised Transfer}


\icmlsetsymbol{equal}{*}

\begin{icmlauthorlist}
\icmlauthor{Artyom Gadetsky}{equal,epfl}
\icmlauthor{Yulun Jiang}{equal,epfl}
\icmlauthor{Maria Brbić}{epfl}\\
\icmlauthor{}{}\\
\weburl
\end{icmlauthorlist}

\icmlaffiliation{epfl}{EPFL, Lausanne, Switzerland}

\icmlcorrespondingauthor{Maria Brbić}{mbrbic@epfl.ch}

\icmlkeywords{Machine Learning, ICML}

\vskip 0.3in
]



\printAffiliationsAndNotice{\icmlEqualContribution}  

\begin{abstract}
Foundation vision-language models have enabled remarkable zero-shot transferability of the pre-trained representations to a wide range of downstream tasks. However, to solve a new task, zero-shot transfer still necessitates human guidance to define visual categories that appear in the data. Here, we show that \textit{fully unsupervised transfer} emerges when searching for the labeling of a dataset that induces maximal margin classifiers in representation spaces of different foundation models. We present \methodnew, a fully unsupervised method that effectively employs this guiding principle to uncover the underlying labeling of a downstream dataset without any supervision and task-specific representation learning. We evaluate \methodnew on a diverse benchmark suite of 26 datasets and show that it achieves new state-of-the-art unsupervised performance. Furthermore, \methodnew, although being fully unsupervised, outperforms zero-shot transfer baselines on a wide range of datasets. In particular, \methodnew matches the average performance of CLIP zero-shot on 26 datasets by employing the same representation space, spanning a wide range of architectures and model sizes. By guiding the search for the underlying labeling using the representation spaces of two foundation models, \methodnew surpasses zero-shot transfer and unsupervised prompt tuning baselines, demonstrating the surprising power and effectiveness of unsupervised transfer.
\end{abstract}
\section{Introduction}
\label{introduction}
Transfer learning is a fundamental machine learning paradigm that leverages large-scale pre-training of deep neural networks to improve model performance on downstream tasks with limited resources \cite{pan2009survey}. Early transfer learning approaches relied on supervised fine-tuning of the entire model to solve a downstream task of interest \cite{kolesnikov2020big}. Recent works \cite{he2022masked, li2022efficient, zhou2022ibot, oquab2023dinov2, darcet2024vision} have shown that fine-tuning an entire model during transfer brings only marginal gains compared to training a linear classifier on top of the frozen pre-trained backbone (\textit{i.e.}, linear probe). Although these approaches eliminated the need for task-specific fine-tuning of representations, they still require at least a few labeled examples per class to achieve human-level performance on downstream tasks.

\begin{figure*}[t!]
    \centering
  \includegraphics[width=0.85\textwidth]{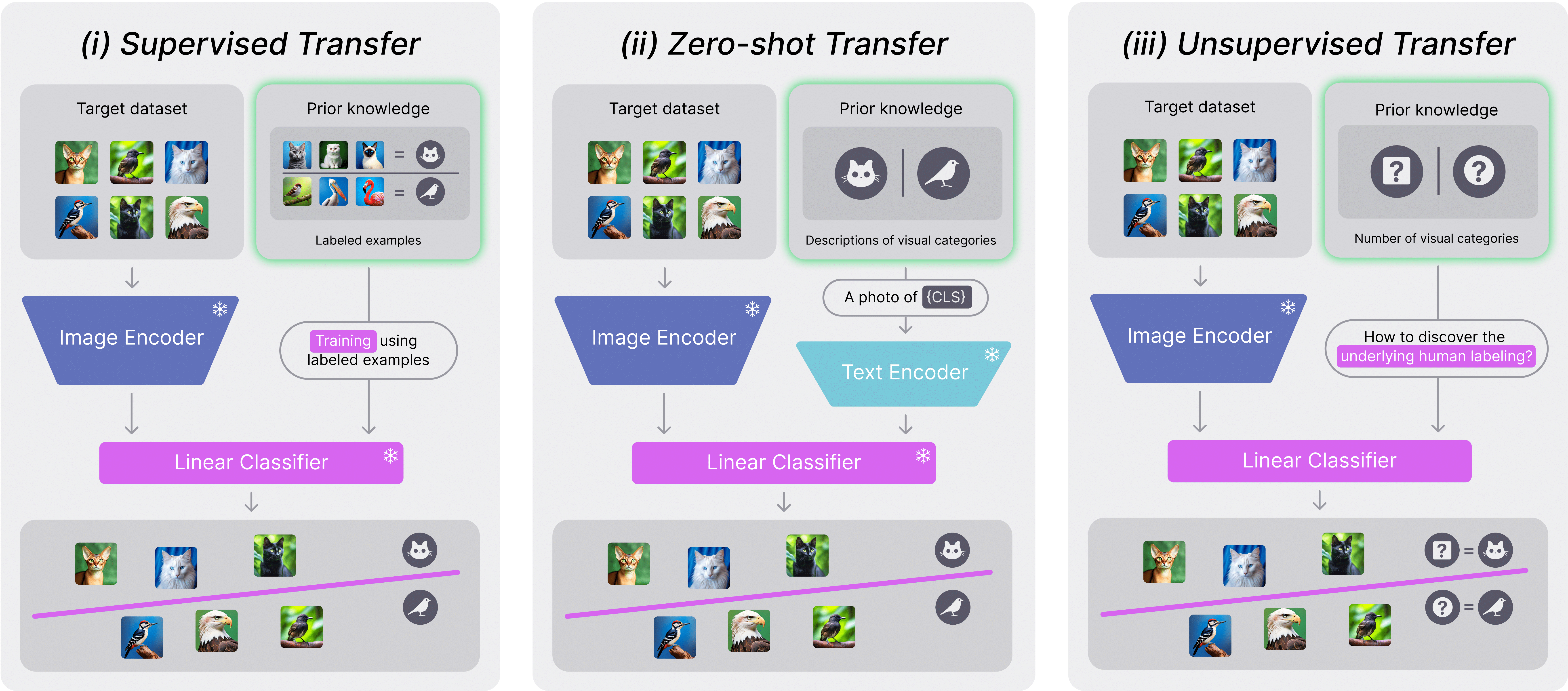}
  \caption{\xhdr{Types of downstream transfer differ in the amount of available supervision} Given representation spaces of foundation models, \textit{(i) supervised transfer}, represented as a linear probe, trains a linear classifier given labeled examples of a downstream dataset; \textit{(ii) zero-shot transfer} assumes descriptions of the visual categories that appear in a downstream dataset are given, and employs them via text encoder to solve the task; and \textit{(iii) unsupervised transfer} assumes the least amount of available supervision, \textit{i.e.}, only the number of categories is given, and aims to uncover the underlying human labeling of a dataset.}
  \label{fig:setting_description}
\end{figure*}

Recently, foundation models \cite{bommasani2022opportunities} have emerged, approaching human-level intelligence on a variety of tasks in the zero-shot setting. In particular, \citet{radford2021learning} proposed CLIP, which trains representations by aligning images and their corresponding captions in the joint embedding space. After pre-training, a zero-shot classifier is constructed by embedding the descriptions of visual categories that appear in the data. Subsequent works have successfully adopted this representation learning principle to enable zero-shot transfer in other domains, such as audio signal processing \cite{elizade2022clap, elizade2023natural}, biomedicine \cite{lin2023pmc, robinson2023contrasting} and symbolic regression \cite{meidani2024snip}. Despite the remarkable success of foundation models, zero-shot transfer still requires human instructions to solve a new task. But, can the representations of foundation models be utilized to solve a new task in a \textit{fully unsupervised manner}?

The simplest approach for unsupervised transfer would be to apply off-the-shelf clustering methods \cite{macqueen1967some} on top of the pre-trained representations. However, this strategy inevitably leads to a drastic decrease in performance compared to (weakly) supervised and zero-shot transfer \cite{zhou2022ibot, oquab2023dinov2}. Recently, \citet{gadetsky2023pursuit} introduced \methodold, an unsupervised learning framework for inferring the underlying human labeling of a given dataset from pre-trained representations. 
While \methodold has achieved superior performance compared to unsupervised baselines, it still requires task-specific representation learning and does not close the gap between unsupervised and zero-shot transfer.

Here, we present \methodnew, a method that enables unsupervised transfer from foundation models. The key idea behind our approach is to search for the labeling of a downstream dataset that maximizes the margins of linear classifiers in the space of single or multiple foundation models to uncover the underlying human labeling. Compared to zero-shot and supervised transfer, unsupervised transfer with \methodnew does not need the supervision in any form (Figure \ref{fig:setting_description}).  Compared to deep clustering methods \cite{xie2016unsupervised,chang2017deep,caron2019deep,van2020scan,niu2022spice}, \methodnew does not require task-specific representation learning that is expensive for modern foundation models.

We study the performance of \methodnew on the extensive evaluation suite spanning $26$ datasets and $7$ different foundation models. We compare \methodnew to various baselines that differ in the amount of available supervision for the downstream transfer. First, when compared to the recent state-of-the-art unsupervised baselines, \methodnew outperforms these baselines on all the considered datasets, setting the new state-of-the-art unsupervised performance. Compared to zero-shot transfer, \methodnew instantiated with two foundation models surpasses CLIP zero-shot transfer across all studied model sizes, achieving exceptional absolute improvements up to $35\%$ on the studied datasets. Given the same single representation space, \methodnew closely matches the performance of the CLIP zero-shot transfer on $7$ out of $8$ studied model architectures. In particular, the best \methodnew model, which utilizes the same model size and representation space, outperforms CLIP zero-shot on $13$ out of $26$ datasets. Finally, when compared to supervised transfer represented by linear probe, \methodnew approaches its performance on 5 out of 26 studied datasets, suggesting that labels may not be needed to infer the underlying human labeling when given sufficiently high-quality representations.
\section{Background}\label{background} 
In this section, we introduce the problem setting of unsupervised transfer and provide an overview of key concepts that we build upon.

\xhdr{Unsupervised transfer} Let $\mathcal{X} \subseteq \mathbb{R}^{d}$ be an input space and $\mathcal{D}=\{x_n\}_{n=1}^{N},\ x_n \in \mathcal{X} $ be a dataset consisting of $N$ samples and $C$ classes, where $C$ is known a priori. Let  $\phi(x):\mathcal{X} \xrightarrow{} \mathbb{R}^{q}$ denotes a mapping from an input space $\mathcal{X}$ to a $q$-dimensional representation space of a pre-trained foundation model. The question we aim to answer is how to utilize representations from foundation models to solve a new task in a \textit{fully unsupervised manner}. Thus, by unsupervised transfer we consider the task of inferring the underlying human labeling\footnote{We interchangeably use terms ``task'' and ``labeling'' in the context of this paper, since any labeling defines a task. Consequently, we refer to a task as human labeled if it corresponds to the underlying human labeling of a given dataset $\mathcal{D}$.} of a dataset $\mathcal{D}$ without any supervision given representations of foundation models.

\xhdr{Generalization-based learning of human labelings} \citet{gadetsky2023pursuit} recently introduced a generalization-based objective that evaluates the generalization ability of linear models on top of representations obtained from pre-trained models. The objective is motivated by a strong generalization ability of linear models in representation spaces of foundation models on many human labeled tasks. Equipped with this insight, the goal is to find such labeling that optimizes generalization ability of a linear model over all possible labelings of a given dataset. The quality of a labeling is measured by the ability of a linear model to generalize on a task defined by the given labeling. 

In particular, let $\tau : \mathcal{X} \xrightarrow{} \{1, \dots, C\}$ denote a labeling function of a dataset. Let $f(x) = w^T \phi(x)$ denote a linear model in the representation space $\phi(x)$ of a foundation model. Given a train-test split $(\mathcal{D}_{tr}, \mathcal{D}_{te})$, one can train the model on a training split $\mathcal{D}_{tr}$ with labeling $\tau(\mathcal{D}_{tr})$ and classification loss function $\mathcal{L}$ to obtain $\hat{f}$. After training, the generalization ability of the model can be assessed by computing the error of $\hat{f}$ on $\mathcal{D}_{te}$. Consequently, the generalization-based objective is defined as follows:
\begin{align}\label{hume_general}
\begin{split}
    \min\limits_{\tau} \sum_{x \in \mathcal{D}_{te}} \mathcal{L}(\hat{f}(x), \tau(x)) & \\  
 \textrm{s.t.}  \  \hat{f} = \arg \min_{f} \sum_{x\in \mathcal{D}_{tr}} & \mathcal{L}(f(x), \tau(x)),
\end{split}
\end{align}
where minimization is performed over the set of all possible labelings of a dataset $\mathcal{D}$. This leads to a difficult discrete optimization problem. To overcome this limitation, \citet{gadetsky2023pursuit} replace minimization \wrt a discrete labeling $\tau$ with minimization \wrt continuous parameters $\theta$ of a task encoder $\tau_{\theta}(x):\mathcal{X} \xrightarrow{} \Delta^{C-1}$, where $\Delta^{C-1}$ denotes $(C-1)$-dimensional probability simplex. As a result, careful design of $\tau_\theta$ becomes crucial since it defines the search space explored by the generalization-based objective (\ref{hume_general}).

\xhdr{\methodold framework} The instantiation of this framework, proposed in HUME \cite{gadetsky2023pursuit}, models $\tau_\theta$ using a linear model in the representation space $\psi(x)$ obtained via self-supervised pre-training on the target dataset $\mathcal{D}$:
\begin{align}\label{hume_task_parametrization}
    \tau_{\theta}^{\text{\methodold}}(x) = \sigma(\theta^T \psi(x)),
\end{align}
where $\sigma:\mathbb{R}\xrightarrow{} \Delta^{C-1}$ denotes an activation function. This modeling choice corresponds to \textit{restricting the search space} in (\ref{hume_general}) to a set of labelings which are linearly separable in the representation space $\psi(x)$. In addition, obtaining $\psi(x)$ \textit{requires task-specific representation learning}, \textit{i.e.}, running self-supervised learning on the target dataset $\mathcal{D}$. Since reliable self-supervised pre-training necessitates a large amount of data \cite{wang2020understanding}, this prevents successful unsupervised transfer on downstream tasks with limited resources. 

Given the task encoder parametrization $\tau_\theta^{\text{\methodold}}$, HUME optimizes the following objective to search for the underlying human labeling:
\begin{align}\label{hume_instantiation}
\begin{split}
        \mathcal{L}^{\text{\methodold}}(\theta) =  & \sum_{x \in \mathcal{D}_{te}} \mathcal{L}_{\text{ce}} (f_{\text{approx}}(x), \tau_\theta^{\text{\methodold}}(x)), \\
\end{split}
\end{align}
where $\mathcal{L}_{\text{ce}}$ is the cross-entropy loss function and $f_{\text{approx}}$ is an approximate solution to $\hat{f}$ obtained using iterative optimization algorithms. HUME resorts to iterative differentiation \cite{domke2012generic, shaban2019truncated} to solve the resulting bilevel optimization problem, \textit{leading to an expensive overall training procedure}.

\section{Analysis of Generalization-Based Objective}\label{analysis}
To understand inductive biases of the generalization-based objective proposed in (\ref{hume_general}), we consider this objective in case of binary labelings $\tau(x) : \mathcal{X} \rightarrow \{-1, +1\}$ with exponential loss function $\mathcal{L}_{\text{exp}}(f(x), \tau(x)) = \exp(-\tau(x) f(x))$. To simplify the analysis, we assume that the task encoder $\tau_{\theta}$ is a linear model in the same representation space $\phi(x)$, \textit{i.e.}, $\tau_{\theta}(x) = \sigma(\theta^T \phi(x))$, where $\sigma: \mathbb{R} \xrightarrow{} [-1; 1]$ is an odd activation function such as $\tanh$. This corresponds to \textit{restricting the search space} in (\ref{hume_general}) to a set of labelings which are linearly separable in the representation space $\phi(x)$. Additionally, we do not distinguish between train and test splits, \textit{i.e.}, $\mathcal{D}_{tr}=\mathcal{D}_{te}=\mathcal{D}$. We provide a detailed discussion of the aforementioned assumptions in the remarks at the end of this section.

To obtain an approximate solution to $\hat{f}$, we use iterative optimization algorithms. Specifically, let $w_{m+1} = \Xi(w_{m}, \mathcal{D})$ denote a one step of an optimization algorithm, \textit{i.e.}, $\Xi(w_{m}, \mathcal{D}) = w_{m} - \eta \nabla_{w} \sum_{x \in \mathcal{D}} \mathcal{L}(w_{m}^T \phi(x), \tau_{\theta}(x))$ for the gradient descent with a step size $\eta$. Similarly, let $w_{M} = \Xi^{(M)}(w_0, \mathcal{D})$ denote $M$ steps of an optimization algorithm starting from $w_{0}$. Eventually, the above specifications result in the following bilevel optimization problem:
\begin{align}
        \mathcal{L}^{\text{binary}}_{M}(\theta) =  & \sum_{x \in \mathcal{D}} \exp(-\tau_\theta(x) w_{M}^T \phi(x)) \label{bin_general_outer} \\
        & \textrm{s.t. } w_M = \Xi^{(M)}(w_0, \mathcal{D}) \label{bin_general_inner},
\end{align}
where we refer to (\ref{bin_general_outer}) and (\ref{bin_general_inner}) as \textit{inner} and \textit{outer} objectives respectively.

The key observation underlying our main result is that the inner optimization (\ref{bin_general_inner}) corresponds to the unregularized logistic regression on \textit{separable} data, allowing us to employ the seminal result by \citet{soudry2018implicit}. This work shows that gradient descent, when applied to the task of unregularized logistic regression, outputs iterates that are biased towards the direction of the max-margin hyperplane.
 Evidently, the task encoder $\tau_\theta$ generates labelings of $\mathcal{D}$, which, by definition, are linearly separable in the representation space $\phi(x)$. Consequently, $w_M$ will follow the direction of max-margin hyperplane for a given labeling $\tau_\theta$. In turn, the last point to observe is that substituting the iterates in (\ref{bin_general_outer}), the outer objective is minimized when $w_M$ has a larger margin $\tau_\theta(x) w_M^T \phi(x)$ with respect to $\tau_\theta$. Equipped with this intuition, we are now ready to state our main result:
\begin{proposition}\label{hume_bound}
    Given $M \gg 1$, $\theta \neq 0$ and appropriate step size $\eta$ which ensures convergence, then 
    \begin{align}\label{humebound}
        \mathcal{L}^{\text{binary}}_{M}(\theta) \geq g(\theta) \|w_{\text{SVM}}(\theta)\|_2^2,
    \end{align}
    where $g(\theta) = (M \eta \exp(\|r_M(\theta)\|_2))^{-1}$, the residual $r_M(\theta)$ is bounded with $\lim_{M \to \infty} \|r_M(\theta)\|_2 = 0$, and $w_{\text{SVM}}(\theta)$ is the solution of the hard-margin SVM for a given $\theta$:
    \begin{align}\label{hardmarginsvm}
    \begin{split}
        w_{\text{SVM}}(\theta) = \min_{w} \quad & \|w\|_2^2 \\
        \textrm{s.t.} \quad & \tau_{\theta}(x_n) w^T \phi(x_n) \geq 1 \quad \forall x_n \in \mathcal{D}.
    \end{split}
    \end{align}
\end{proposition}
We defer the proof to Appendix \ref{appendix:proofs}. This result shows that the generalization-based objective upper bounds the norm of hard-margin SVM fitted to a labeling $\tau_\theta$. Consequently, minimizing $\mathcal{L}^{\text{binary}}_{M}$ will inevitably lead to minimizing the norm (\textit{i.e.}, maximizing the margin) \textit{with respect to a labeling}. As a result, the optimization procedure will yield labelings with large margin of the corresponding classifier. Overall, our result unveils that the maximum margin principle \cite{vapnik1995nature}, widely employed by supervised learning algorithms, emerges as the inductive bias of the generalization-based objective (\ref{hume_general}).

\begin{remark}\label{linearseparabilityremark} \textit{(Search space restriction).}
    The result above holds when labelings generated by $\tau_\theta$ are linearly separable in the representation space $\phi(x)$. This assumption leads to the analysis of the generalization-based objective (\ref{hume_general}) with the restricted search space. \citet{ji2019implicit} showed that in the case of non-separable labelings, gradient descent mirrors the separable case, following the max-margin direction of a maximal linearly separable subset of the data. Therefore, one could expect that the lower bound of the generalization-based objective (\ref{hume_general}) optimized over the complete search space inherits these properties, reflecting the separable case. 
\end{remark}
\begin{remark}\label{samplingtraintestremark} \textit{(Train-test split assumption).}
    The generalization-based objective (\ref{hume_general}) assumes different train-test splits $(\mathcal{D}_{tr}, \mathcal{D}_{te})$ on the inner-outer levels respectively to obtain an unbiased estimate of the true risk of a model $f$. In our analysis, we simplify this assumption and employ $\mathcal{D}$ on both levels. Our result shows that minimizing the generalization-based objective in this case leads to maximizing the margin of a linear model with respect to a labeling $\tau$ on $\mathcal{D}$. In turn, this will inevitably lead to low error on a held out data given that margin size upper bounds generalization error \cite{bartlett1999generalization, gronlund2020near}.
\end{remark}
\begin{remark}\label{asymptoticanalysisremark} \textit{(Asymptotic analysis)} Proposition \ref{hume_bound} is rather informal since it substitutes the asymptotic behaviour of the gradient descent iterates $w_{M}$ into the outer objective. Although a rigorous analysis of the residual is required to establish exact bounds, these results serve to grasp the inductive bias incorporated in the generalization-based objective designed for the inference of human labelings.
\end{remark}

In summary, this result shows that optimizing the generalization-based objective (\ref{hume_general}) yields labelings that induce maximal margin classifiers in the representation space $\phi(x)$. Our main result is greatly inspired by the seminal works \cite{soudry2018implicit, ji2019implicit} that reveal the implicit bias of gradient descent towards max-margin solution. Likewise, we demonstrate that the generalization-based objective (\ref{hume_general}) encourages labelings $\tau$ such that if one were to subsequently train a max-margin classifier in the representation space $\phi(x)$ to fit a labeling $\tau$, \textit{the margin obtained would be maximal over all possible labelings}. 
\section{\methodnew Framework}\label{method}
These insights serve us as a guiding principle to develop \methodnew, a general framework for efficient \textit{fully unsupervised transfer} given representations of foundation models.

\xhdr{Optimization objective} Proposition \ref{hume_bound} provides an important insight on the inductive bias incorporated in the generalization-based objective (\ref{hume_general}). Indeed, one can search for the underlying human labeling by maximizing the margin of a linear model with respect to a labeling. Pushing the limits of this principle, we propose to search for a labeling by maximizing margins of linear models \textit{in spaces of multiple foundation models at the same time}. Given $K$ foundation models, let $\phi_k(x)$ be a representation space of $k\text{-th}$ foundation model. Given labeling defined by a task encoder $\tau_{\theta}$, let $w^k_{M}$ be $k\text{-th}$ linear model trained to fit this labeling in a representation space $\phi_k(x)$. Then, \methodnew's optimization objective is as follows:
\begin{align}\label{turtle_instantiation}
\begin{split}
    \mathcal{L}_M^{\text{\methodnew}}(\theta) = & \sum_{k=1}^{K} \sum_{x \in \mathcal{D}} \mathcal{L}_{\text{ce}}(w^k_M \phi_k(x); \tau_\theta(x)) \\
    & \textrm{s.t. } w^k_M = \Xi^{(M)}(w^k_0, \mathcal{D}), \forall k,
\end{split}
\end{align}
where, $\Xi^{M}(w^k_0, \mathcal{D})$ denotes an iterative optimization algorithm $\Xi$ run for $M$ steps starting from $w^k_0$. Intuitively, each of the $K$ terms in the loss function encourages $\tau_\theta$ to maximize margin of $k\text{-th}$ linear model in the corresponding representation space $\phi_k$.  As opposed to the \methodold's objective (\ref{hume_instantiation}), which maximizes margin only in the single space $\psi(x)$, \methodnew provides more effective guidance to the search process.

\xhdr{Task encoder parametrization} The parametrization of a task encoder $\tau_\theta$  defines the search space of labelings, thus it has a crucial importance on the optimization process. In \methodnew, we employ pre-trained representation spaces of foundation models to define a task encoder $\tau_\theta$. These representations \textit{remain fixed} during the overall training procedure, alleviating the need of task-specific representation learning. 

In particular, given $K$ representation spaces $\phi_k(x)$, we define our task encoder $\tau_\theta$ as follows:
\begin{align}\label{turtle_task_parametrization}
\begin{split}
    \tau_{\theta}^{\text{\methodnew}}(x) & = \frac{1}{K} \sum_{k=1}^{K} \tau_{\theta_k}(x), \\
    \textrm{where } \tau_{\theta_k}(x) & = \sigma(\theta_k^T \phi_k(x)),
\end{split}
\end{align}
such that $\theta = \{\theta_1, \dots, \theta_K\}$ denotes all trainable parameters and $\sigma$ is a softmax activation function. After training, cluster assignments are computed as usual:
\begin{align}\label{turtle_cluster_assignments}
 \arg \max\limits_{c = 1, \dots, C} \left[\tau_{\theta}^{\text{\methodnew}}(x)\right]_c,
\end{align}
where $\left[\tau_{\theta}^{\text{\methodnew}}(x)\right]_c$ denotes the probability of assigning a sample $x$ to the $c$-th cluster.

Compared to the \methodold framework in (\ref{hume_task_parametrization}) which searches for the underlying human labeling only over all linearly separable labelings in the self-supervised representation space $\psi(x)$, \methodnew's parametrization greatly expands the search space. Indeed, modeling $\tau_\theta$ as a simple ensemble induces the search space which is at least union of all linearly separable labelings in each of the representation spaces of foundation models $\phi_1, \dots, \phi_K$. One could further suggest employing deeper architectures to model $\tau_\theta$, however such modeling choice may give rise to tasks that capture spurious correlations in data and do not necessarily reflect human labelings \cite{atanov2022task}. Therefore, our design choice effectively increases the search space and alleviates the need of task-specific fine-tuning by employing strong representations of foundation models.

\xhdr{Regularization} The task encoder can synthesize degenerate labelings, \textit{i.e.}, assign all samples to a single class \cite{gadetsky2023pursuit}. Although such labelings induce linear classifiers with the largest possible margin in all representation spaces, they are irrelevant. To avoid such trivial solutions, we separately regularize each term of the task encoder:
\begin{align}\label{turtle_entropy_regularization}
    \mathcal{R}(\theta) = \sum_{k=1}^{K} \mathbb{H}(\overline{\tau}_{\theta_k}^{k}),
\end{align}
where $\overline{\tau}_{\theta_k}^{k} = (|\mathcal{D}|)^{-1} \sum_{x \in \mathcal{D}} \tau_{\theta_k}(x) \in \Delta^{C-1}$ is an empirical label distribution of $k$-th component $\tau_{\theta_k}$ and $\mathbb{H}(\cdot)$ is the entropy function of discrete distribution. 

\xhdr{Final objective function} Putting  (\ref{turtle_instantiation}) and (\ref{turtle_entropy_regularization}) together, \methodnew finally optimizes the following objective function:
\begin{align}\label{turtle_final_objective}
    \min\limits_{\theta} \mathcal{L}_M^{\text{\methodnew}}(\theta) - \gamma \mathcal{R}(\theta),
\end{align}
where we found $\gamma = 10$ is a good default choice for the entropy regularization strength $\gamma$. We show robustness to this hyperparameter in Appendix \ref{app:entropyablation}.

\xhdr{Efficient optimization} The new optimization-based objective (\ref{turtle_instantiation}) is a bilevel optimization problem with the convex inner part. Indeed, given $\tau_\theta$, computing $w^{k}_{M}$ corresponds to the logistic regression problem on $\mathcal{D}$ with labeling $\tau_{\theta}(\mathcal{D})$ in the $k$-th representation space $\phi_k$. Learning parameters $\theta$ using gradient-based techniques involves computing a total derivative $\frac{d }{d \theta} \mathcal{L}^{\text{\methodnew}}_{M}$:
\begin{align}\label{turtle_approx_hypergrad}
    \frac{d}{d \theta} \mathcal{L}^{\text{\methodnew}}_{M} = \frac{\partial }{\partial \theta} \mathcal{L}^{\text{\methodnew}}_{M} + \sum_{k=1}^{K} (\frac{\partial w_M^k}{\partial \theta})^T \frac{\partial }{\partial w_M^k} \mathcal{L}^{\text{\methodnew}}_{M}, 
\end{align}
where $\frac{\partial w_M^k}{\partial \theta}$ is the Jacobian, which is expensive to compute in practice \cite{domke2012generic, shaban2019truncated}. The key observation is that employing the same set of samples $\mathcal{D}$ on both inner and outer levels allows us to discard the second term of the total derivative. Indeed, after training $w^{k}_{M}$ on $\mathcal{D}$, one can approximate $\frac{d}{d \theta} \mathcal{L}^{\text{\methodnew}}_{M} \approx \frac{\partial }{\partial \theta} \mathcal{L}^{\text{\methodnew}}_{M}$ since $w^k_{M}$ is an approximate stationary point of the inner problem for a given $\tau_\theta$, \textit{i.e.}, $\frac{\partial }{\partial w_M^k} \mathcal{L}^{\text{\methodnew}}_{M} \approx 0$. \citet{ablin2020super} have shown a strong performance of this estimator in practice for bilevel optimization problems similar to ours. The pseudocode of \methodnew is provided in Algorithm~\ref{alg:turtle} with implementation details in Appendix \ref{app:implementationdetails}.

\section{Experiments}\label{experiments}
\subsection{Experimental setup}

\xhdr{Datasets and evaluation metric} We study the performance of \methodnew on the extensive benchmark of 26 vision datasets \cite{radford2021learning}. The detailed description of each dataset is provided in Appendix \ref{app:datasets}. We compare our framework with the baselines using accuracy metric and employ Hungarian algorithm \cite{kuhn1955hungarian} to match the labeling found by \methodnew (\ref{turtle_cluster_assignments}) to the ground truth labeling of a corresponding dataset. By default, we train \methodnew on the training split of a corresponding dataset and provide the results on the test split. In Appendix \ref{app:testtestexp}, we additionally show that mimicking deployment regime, \textit{i.e.}, having only test split available for training, does not lead to performance decrease of \methodnew.

\xhdr{Foundation models in \methodnew} We employ CLIP \cite{radford2021learning} representations which span different architectures and model sizes, in particular, $5$ different ResNets (R50, R101, R50x4, R50x16 and R50x64) and $3$ different Vision Transformers (ViT-B/32, ViT-B/16 and ViT-L/14). We refer to the \methodnew as \methodnew 1-space if it utilizes only a single space CLIP representation ($K=1$ in (\ref{turtle_instantiation}) and (\ref{turtle_task_parametrization})). We refer to the  \methodnew as \methodnew 2-spaces if it utilizes two different foundation models. Namely, we use  DINOv2 ViT-g/14 \cite{oquab2023dinov2} as the second space while the first space is always represented with one of the CLIP variants. Consequently, to specify the particular CLIP architecture when utilizing two representation spaces, \textit{e.g.}, ViT-L/14, we refer to \methodnew as \methodnew 2-spaces ViT-L/14.  We precompute all representations for the entire benchmark and keep these representations fixed during the overall training procedure. The detailed description of the used models and other specifications to prepare representations are provided in Appendix \ref{app:representations}.

\xhdr{Baselines} We compare unsupervised transfer using \methodnew to baselines that differ in the amount of supervision they use (Figure \ref{fig:setting_description}). First, we compare \methodnew to \methodold \cite{gadetsky2023pursuit}, a method that has recently shown state-of-the-art unsupervised learning performance and surpassed traditional deep clustering approaches \cite{van2020scan,niu2022spice,amrani2022self,feng2023self}. Next, to explore how far can we go with unsupervised transfer, we compare \methodnew in a challenging setting to zero-shot transfer, unsupervised prompt tuning and supervised baselines. All these baselines use some form of supervision compared to TURTLE which is fully unsupervised. We start by comparing \methodnew to the CLIP zero-shot transfer \cite{radford2021learning} that employs descriptions of ground truth classes as a form of supervision. Following \cite{radford2021learning}, we perform prompt engineering and ensembling to construct a zero-shot classifier for each dataset. As even stronger baselines, we compare \methodnew to the state-of-the-art unsupervised prompt tuning methods UPL \cite{huang2022unsupervised}, POUF \cite{tanwisuth2023pouf} and GDA \cite{wang2024hard}. These approaches enhance class prototypes defined by the CLIP zero-shot classifier via unsupervised adaptation on the downstream task. Finally, we employ supervised linear probe on top of the CLIP representations to serve as a supervised transfer baseline. Differences between types of transfer are highlighted in Table \ref{tab:exp_setting_differences}.

\begin{table}[ht]
\caption{\xhdr{Differences between the considered types of downstream transfer}}
\label{tab:exp_setting_differences}
\resizebox{\columnwidth}{!}{%
\begin{tabular}{lcc}
\toprule
& Available Supervision & Training on $\mathcal{D}$ \\
\midrule
Unsupervised transfer (\textit{ours}) & Number of classes & \cmark \\
Zero-shot transfer & Class descriptions & \xmark \\
Unsupervised prompt tuning & Class descriptions & \cmark \\
Supervised transfer & Labeled samples & \cmark \\
\bottomrule
\end{tabular}
}
\end{table}

\xhdr{Model Selection}
\citet{gadetsky2023pursuit} showed that generalization-based objective (\ref{hume_general}) is strikingly well-correlated with human labelings, which we further confirm in Figure \ref{fig:cross_val} on 26 datasets. Notably, this enables \textit{unsupervised hyperparameter search} in \methodnew. For supervised linear probes, we perform standard cross-validation to search for the L2-regularization strength. We refer the reader to Appendix \ref{app:implementationdetails} for the detailed description of our model selection procedures. Code is publicly available at \url{https://github.com/mlbio-epfl/turtle}.

\subsection{Results}
\xhdr{Comparison to unsupervised baselines} We start by comparing \methodnew to \methodold. Originally, \methodold utilized self-supervised representation learning on the given dataset $\mathcal{D}$ to model the task encoder (\ref{hume_task_parametrization}). To ensure the fair comparison, we instead employ representation spaces of foundation models for both modeling the task encoder (\ref{hume_task_parametrization}) and for representation space used to model $f_{\text{approx}}$ in (\ref{hume_instantiation}). Consequently, both \methodnew and \methodold use the same representation spaces, \textit{i.e.}, CLIP ViT-L/14 and DINOv2. Furthermore, we improve the optimization procedure of \methodold to enable accelerated convergence. In addition, we compare TURTLE to the K-Means clustering \cite{macqueen1967some} on top of concatenated embeddings from both representation spaces employed by \methodnew. The K-Means clustering serves as the simple unsupervised transfer baseline since, like \methodnew, it does not require task-specific representation learning. We refer the readers to Appendix \ref{appendix:unsuperviseddetails} for the detailed description of the improvements made to \methodold as well as the implementation details of the K-Means.

As shown in Figure \ref{fig:hume_vs_turtle_acc}, \methodnew substiantially outperforms \methodold on all considered datasets, confirming that maximizing margin in both spaces simultaneously to search for the underlying human labeling (\ref{turtle_instantiation}) and expanding the search space of labelings (\ref{turtle_task_parametrization}) is the effective design choice. Remarkably, TURTLE leads to $23\%$ and $11\%$ absolute improvements ($30\%$ and $18\%$ relative improvement) on the MNIST and Birdsnap datasets, respectively. Furthermore, among other datasets, \methodnew sets the new state-of-the-art unsupervised performance on the ImageNet dataset, achieving $72.9\%$ accuracy and outperforming the previous state-of-the-art \cite{alkin2024mim} by $5.5\%$.

\begin{figure}[ht]
\centering
\includegraphics[width = \linewidth]{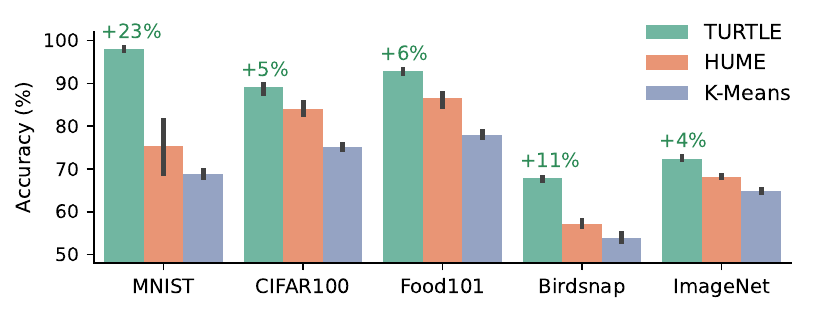}
\vspace*{-8mm}
\caption{
\textbf{\methodnew outperforms unsupervised baselines.} Comparison of TURTLE to unsupervised baselines with respect to accuracy. All methods use the CLIP ViT L/14 and DINOv2 representations. Bars represent the average performance with standard deviations computed over three runs.
}
\label{fig:hume_vs_turtle_acc}
\end{figure}

In addition, we validate optimization efficiency of \methodnew and compare training time between all the considered methods in Figure \ref{fig:hume_vs_turtle_time}. The results corroborate the use of first-order hypergradient approximation (\ref{turtle_approx_hypergrad}) in \methodnew. Notably, \methodnew achieves $10\times$ speedup compared to \methodold, achieving the impressive training time on the ImageNet dataset that takes less than five minutes. Overall, our results show that \methodnew effectively addresses the challenges of unsupervised transfer and outperforms unsupervised baselines by a large margin.

\begin{figure}[h]
\vspace*{-3mm}
\centering
\includegraphics[width=\linewidth]{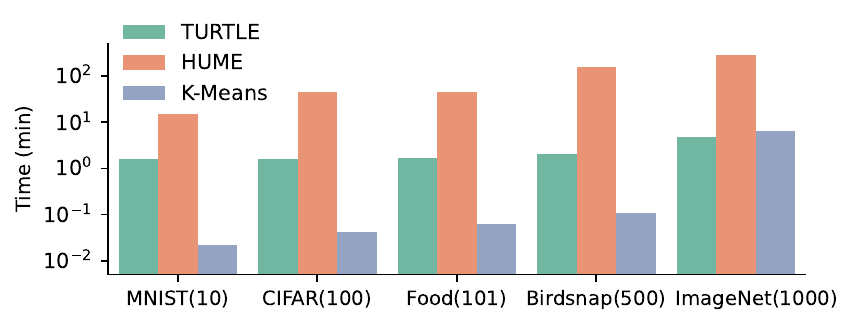}
\vspace*{-9mm}
\caption{
\textbf{\methodnew is an efficient method.} Comparison of running time between TURTLE and unsupervised baselines. TURTLE employs efficient first-order optimization procedure, achieving more than \textbf{10$\times$ speedup} compared to HUME.  All methods use CLIP ViT L/14 and DINOv2 representations. Bars represent the average performance over three runs. Standard deviations are negligible (Table \ref{tab:unsupervised_time}) and omitted for clarity.}
\label{fig:hume_vs_turtle_time}
\end{figure}

\xhdr{Comparison to zero-shot transfer} We compare \methodnew to the CLIP zero-shot transfer that uses descriptions of ground truth classes as a form of supervision. Remarkably, without using any supervision, \methodnew 2-spaces outperforms the zero-shot transfer of CLIP by a large margin across 26 benchmark datasets for different ViT backbones (Figure \ref{fig:intro_plot_scaling_law_turtle_vits}). 

\begin{figure}[h]
\vspace*{-3mm}
\centering
\includegraphics[width =\linewidth]{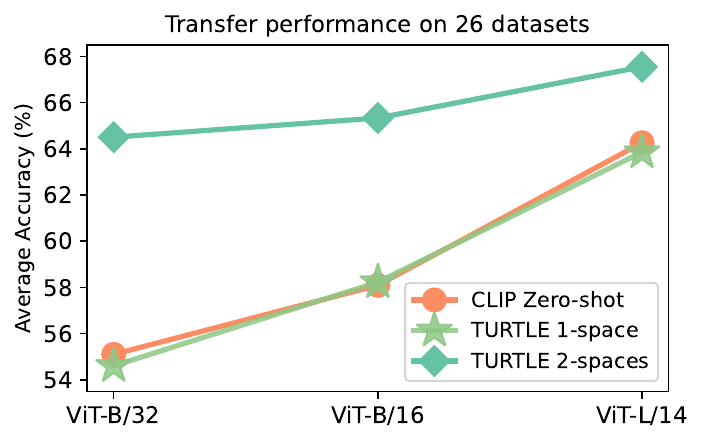}
\vspace*{-9mm}
\caption{\xhdr{\methodnew enables unsupervised transfer given representation spaces of foundation models} Employing the same CLIP representation space, \methodnew closely matches the performance of the corresponding CLIP zero-shot classifier on average over 26 datasets. With the use of an additional representation space, \methodnew outperforms zero-shot transfer, demonstrating exceptional abilities of unsupervised transfer learning.}
\label{fig:intro_plot_scaling_law_turtle_vits}
\end{figure}

\begin{figure}[t]
\centering
\includegraphics[width = \linewidth]{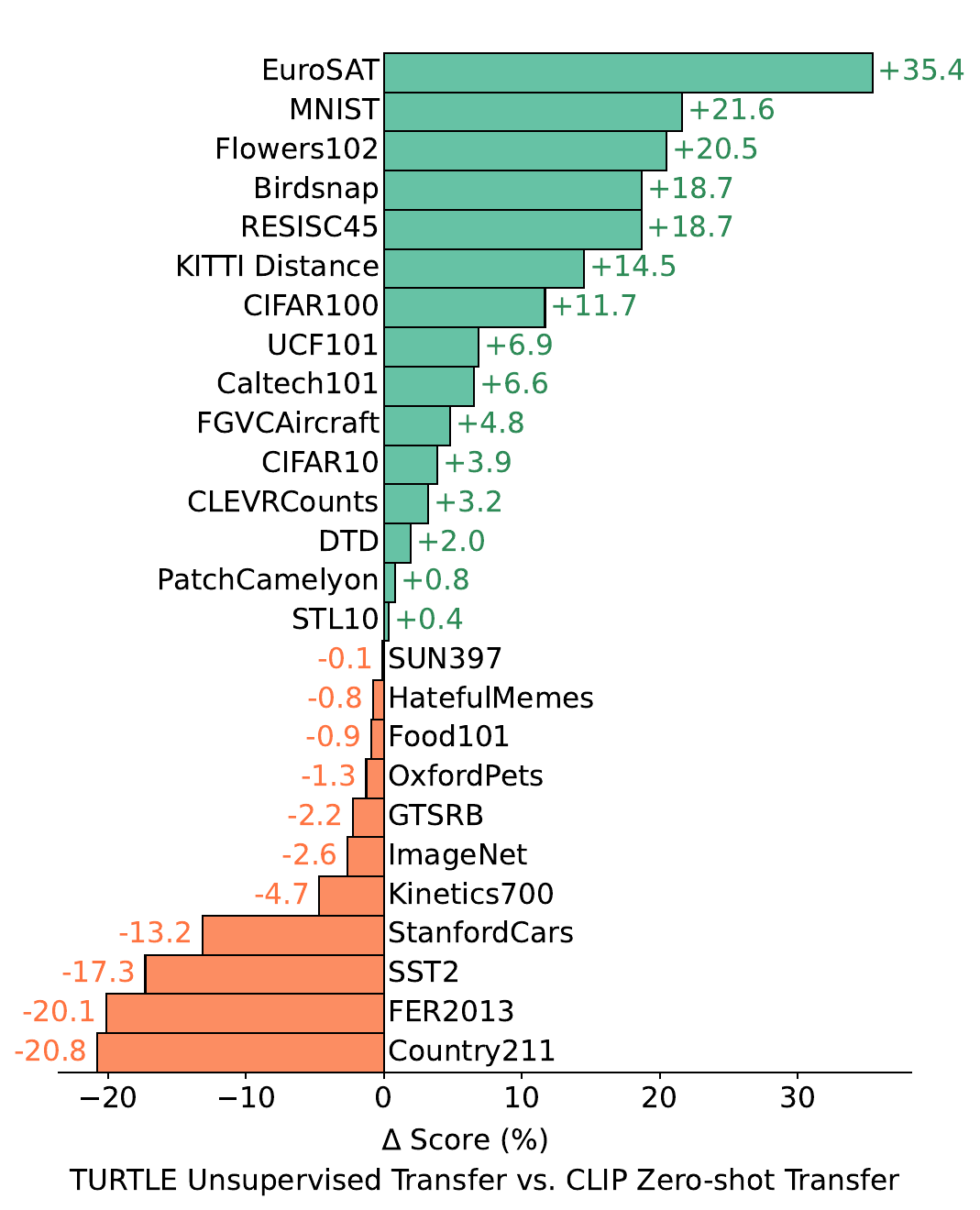}
\vspace*{-9mm}
\caption{\xhdr{\methodnew outperforms the CLIP zero-shot classifier on 15 out of 26 datasets} TURTLE is trained with CLIP ViT-L/14 and DINOv2 representations. CLIP zero-shot utilizes the same CLIP ViT-L/14 architecture. Furthermore, we observe that even with only a single CLIP representation space TURTLE outperforms CLIP on $13/26$ datasets (Figure \ref{fig:turtle_clip_zs}).} \label{fig:turtle_clipdino_zs}
\vspace*{-7mm}
\end{figure}

In particular, \methodnew 2-spaces outperforms CLIP zero-shot by $9\%$, $7\%$ and $4\%$ absolute improvement ($17\%$, $12\%$ and $5\%$ relative improvement) with ViT-B/32, ViT-B/16 and ViT-L/14 backbones, respectively.  Moreover, even \methodnew  1-space matches the performance of CLIP zero-shot across all studied ViT models.  It is important to note that both CLIP zero-shot and \methodnew 1-space are linear models in the same representation space and \textit{differ only in the amount of supervision} which is available to produce the weights. When comparing performance on individual datasets, \methodnew outperforms CLIP zero-shot transfer on $15$ out of $26$ datasets with remarkable absolute gains of $35\%$, $21\%$ and $20\%$ on the EuroSAT, MNIST and Flowers102 datasets, respectively (Figure \ref{fig:turtle_clipdino_zs}). We provide individual scores for all \methodnew and CLIP zero-shot variants in Appendix \ref{appendix:allnumericresults}. 

\begin{table*}[t]
\centering
\caption{\xhdr{\methodnew 2-spaces outperforms unsupervised prompt tuning methods} \textit{ZS} column indicates whether method utilizes zero-shot supervision to make predictions. All methods employ CLIP ResNet-50 representations. \methodnew additionally uses DINOv2 representations as the second representation space.}
\tabcolsep=0.1cm
\begin{tabular}{cc|rccccccccccc}
\midrule[1pt]
 \bf Method & \bf \ \ ZS\ \ & \ \ \bf Pets  & \bf  Flowers & \bf  FGVC & \bf  DTD & \bf  EuroSAT & \bf  Cars & \bf  Food & \bf SUN & \bf  Caltech & \bf  UCF & \bf  ImageNet & \bf   Avg.\\
 \midrule
POUF &\cmark & 88.0 & 66.7 & 16.7 & 41.5 & 42.1 & 57.4 & 74.7 & 58.6 & 86.9 & 61.1 & 55.2 & 59.0 \\ 
UPL &\cmark & 88.3 & 68.9 & 17.3 & 46.6 & 54.8 & \bf 62.1 & 77.6 & 64.0 & \bf89.9 & 67.2 & 60.5 & 
 63.4 \\
GDA  &\cmark & 89.9 & 72.7 & 18.7 & 46.8 & 49.9 & 60.8 & 78.3 & 63.6 & 87.5 & 68.7 & 61.2 & 63.5 \\
\midrule
TURTLE  &\xmark &\bf90.9 &\bf99.7	&\bf25.3	&\bf57.0	&\bf95.5	&32.6	&\bf84.1	&\bf65.7	&88.6	&\bf77.7	&\bf66.3	&\bf71.2 \\
\midrule[1pt]
\end{tabular}
\label{tab:prompttuningtable}
\vspace*{-3mm}
\end{table*}

\xhdr{Comparison to unsupervised prompt tuning} Next, we compare \methodnew to unsupervised prompt tuning baselines. We follow previous works and use CLIP ResNet-50 representations for all methods. Although being fully unsupervised, \methodnew consistently outperforms all the considered baselines by a large margin (Table \ref{tab:prompttuningtable}). Specifically, \methodnew achieves $8\%$ absolute improvement ($12\%$ relative improvement) in average accuracy over the best unsupervised prompt tuning baseline. On the Flowers102 and EuroSAT datasets, our framework attains outstanding absolute gains of $27\%$ and $41\%$ ($37\%$ and $75\%$ relative improvement), respectively. Overall, these results demonstrate the surprising effectiveness of the unsupervised transfer. 

\xhdr{Comparison to supervised transfer} Finally, we compare \methodnew 1-space ViT-L/14 to supervised linear probe in the same representation space. This means that in this setup both models are linear in the representation space of CLIP ViT-L/14 and differ only in the amount of supervision utilized to produce the weights. Supervised linear probe is trained using all available labels. Consequently, we can assume that it represents the maximal transfer learning performance that can be achieved by the unsupervised transfer. 
 
We observe a high positive correlation of $0.87$ (p-value $< 10^{-8}$) between unsupervised transfer performance and its fully supervised counterpart (Figure \ref{turtle_linear_datasets}). This result indicates that with better supervised linear probe performance, \methodnew's performance may also increase, which we further investigate in the subsequent paragraph. Notably, \methodnew approaches the ``optimal`` transfer performance on the STL10, CIFAR10, Flowers102, Food101 and HatefulMemes, demonstrating that \textit{labels may not be needed} when given sufficiently high-quality representations, as measured by supervised linear probe. We perform similar analysis for \methodnew 2-spaces and observe stronger correlation, leading to reduced gap between \methodnew 2-spaces and supervised linear probe (Figure \ref{turtle_linear_datasets_multiple}).

\begin{figure}[t]
\centering
\includegraphics[width = \linewidth]{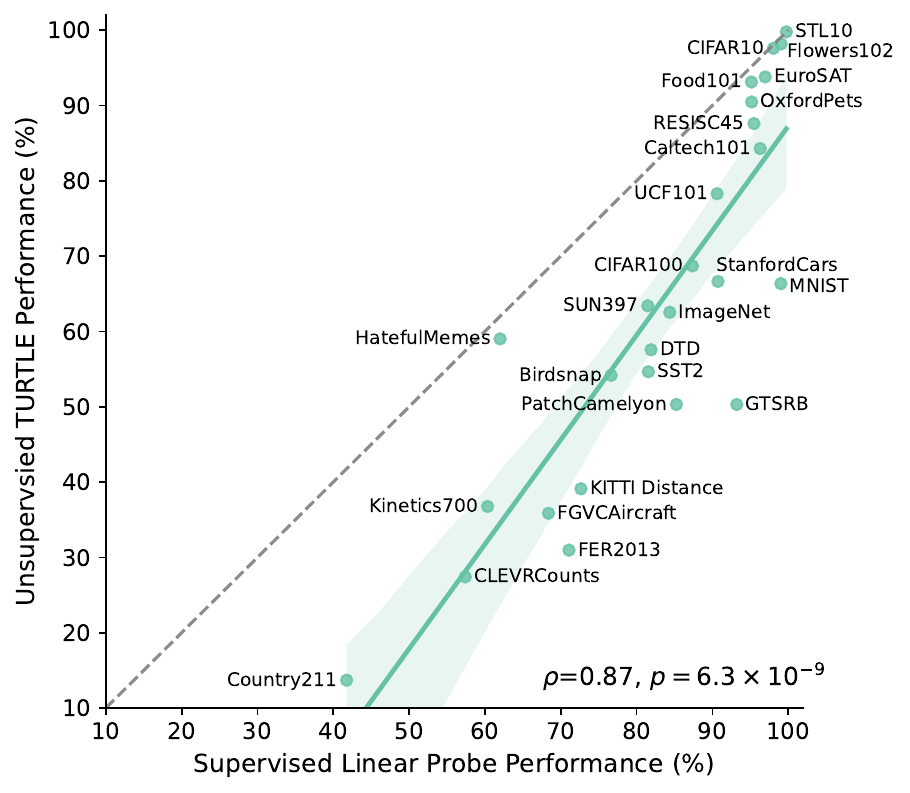}
\vspace*{-10mm}
\caption{\xhdr{Unsupervised transfer performance of \methodnew is correlated with supervised linear probe performance} Dashed line $y=x$ denotes the ``optimal" unsupervised transfer. The performance of \methodnew and supervised linear probe shows a strong correlation ($\rho=0.87, p=6.3\times 10^{-9}$ of two-sided Pearson correlation coefficient). On $5$ datasets \methodnew approaches the performance of the ``optimal" unsupervised transfer ($\leq 3$ point difference).} \label{turtle_linear_datasets}
\vspace*{-5mm}
\end{figure}

\begin{figure}[t]
\centering
\includegraphics[width = 0.65\linewidth]{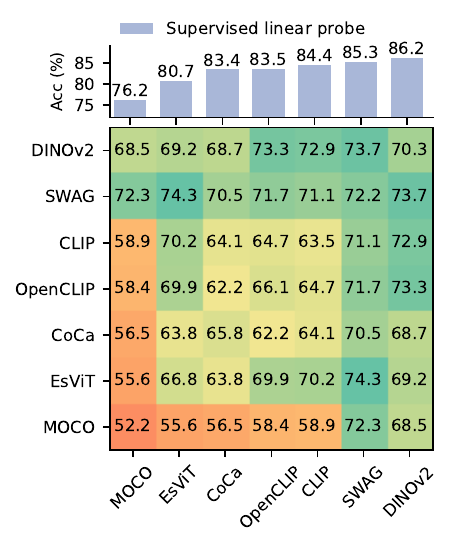}
\vspace{-7mm}
\caption{\textbf{Top}: Supervised linear probe on the ImageNet-1000 dataset for $7$ different representation spaces. \textbf{Bottom}: Heat map represents unsupervised performance of \methodnew on ImageNet-1000. Secondary diagonal cells correspond to \methodnew 1-space, while off-diagonal cells refer to \methodnew 2-spaces with the pair of corresponding representation spaces. The performance of \methodnew indicates a strong positive correlation with the performance of supervised linear probe ($\rho=0.74, p=1.4\times 10^{-9}$ of two-sided Pearson correlation coefficient).} 
\label{fig:turtle_imagenet}
\vspace*{-8mm}
\end{figure}

\xhdr{Ablation of different representation spaces on ImageNet} Results from the previous paragraph speculate that incorporating stronger representations may lead to the increased performance of unsupervised transfer. To validate this, we run \methodnew with pairs of different representation spaces on the ImageNet-1000 dataset \cite{deng2009imagenet}. Results in Figure \ref{fig:turtle_imagenet} show a positive correlation of $0.74$ (p-value $< 10^{-8}$) between unsupervised transfer performance and the quality of representations measured by supervised linear probe. The obtained result confirms that employing stronger representations for a given dataset leads to the improved performance of \methodnew. Consequently, \methodnew can further improve performance by exploiting continual progress in the development of foundation models. Furthermore, given high positive correlation between \methodnew's accuracy and the generalization-based objective (Figure \ref{fig:cross_val}), \methodnew can be utilized as the proxy to measure the quality of given representations in the absence of labels for the downstream task. 
\section{Related Work}\label{relatedwork}

\xhdr{(Weakly) supervised transfer} (Weakly) supervised transfer approaches require at least some amount of supervision to perform downstream transfer. For instance, BigTransfer \cite{kolesnikov2020big} showed that supervised fine-tuning of the entire model after large-scale pre-training successfully transfers knowledge in both fully supervised and few-shot regimes. Recent advances in self-supervised learning \cite{he2022masked, li2022efficient, zhou2022ibot, oquab2023dinov2, darcet2024vision} have demonstrated that a linear probe suffices to achieve competitive performance compared to the fine-tuning the entire model. Despite the strength of these approaches, they necessitate labeled examples to perform downstream transfer. 

\xhdr{Zero-shot transfer} Foundation models such as CLIP \cite{radford2021learning} have recently enabled zero-shot transfer, which relies only on a set of human instructions such as descriptions of visual categories that appear in the data rather than a set of labeled examples. Despite the success of zero-shot transfer in different domains \cite{elizade2022clap, elizade2023natural, lin2023pmc, robinson2023contrasting, meidani2024snip}, collecting zero-shot annotations still requires expert domain knowledge which can be hard to get in many real-world applications. In contrast to the zero-shot transfer approaches, \methodnew enables \textit{fully unsupervised transfer}, effectively alleviating the need of any human guidance. 

\xhdr{Deep clustering} Deep clustering methods \cite{xie2016unsupervised,chang2017deep, caron2019deep, van2020scan, niu2022spice} aim to jointly perform deep representation learning and clustering on a target dataset. Recent state-of-the-art approaches \cite{van2020scan, niu2022spice} rely on time-consuming three-stage procedures that involve self-supervised representation learning, clustering and fine-tuning via self-labeling respectively. In contrast to the deep clustering approaches, \methodnew alleviates the need for laborious task-specific representation learning by employing representation spaces of pre-trained foundation models. Furthermore, compared to deep clustering methods that heavily depend on image augmentations to induce semantically meaningful clusters, \methodnew builds upon the seminal maximum margin principle that is effortlessly applicable beyond image data modality. Consequently, our approach offers an efficient and effective way to perform fully unsupervised transfer from foundation models.

\xhdr{Maximum margin clustering} Our work has revealed that optimizing the generalization-based objective proposed in \citet{gadetsky2023pursuit} results in the search for a labeling that maximizes the margin of a maximal margin classifier over all possible labelings of a dataset. The first attempt to employ maximum margin principle to perform clustering dates back to Maximum Margin Clustering (MMC) \cite{xu2004maximum}. Later works extended this framework to multi-class clustering \cite{xu2005unsupervised, wang2010linear}, multi-view clustering \cite{zhao2009multiple}, or focused on improving the scalability \cite{zhang2007maximum, wang2010linear}. Compared to \methodnew, which employs efficient first-order gradient optimization techniques, the aforementioned approaches rely on the expensive discrete optimization techniques. Furthermore, each of the approaches adopts maximum margin principle in its own way to enable multi-class or multi-space scenario, while \methodnew provides a unified framework for any number of classes and representation spaces.

\xhdr{Implicit bias of optimization algorithms} Understanding the implicit bias of optimization algorithms plays a crucial role in modern machine learning. The seminal work by \citet{soudry2018implicit} showed that the gradient descent, when applied to the task of unregularized logistic regression on separable data, converges to the direction of the maximal margin hyperplane without explicitly enforcing such margin maximization. Later, \citet{ji2019implicit} extended the analysis and demonstrated a similar behavior of gradient descent in the case of non-separable data. In our work, we employ the aforementioned findings to study the inductive bias of the generalization-based objective. Surprisingly, we reveal that it yields labelings that maximize the margin of a maximal margin classifier with respect to labeling. As a result, this insight allows us to develop \methodnew, a method that enables fully unsupervised transfer given representations of foundation models.
\section{Conclusion} \label{conclusion}
In this work, we have shown that the representations of foundation models can be utilized to solve a new task in a fully unsupervised manner. The key insight behind our approach is to search for a labeling that induces maximal margin classifiers in the representation spaces of foundation models. We utilize this insight to develop \methodnew, a general framework for effective unsupervised transfer given representations of different foundation models. Through extensive evaluation, we found that \methodnew, being fully unsupervised, achieves competitive performance compared to zero-shot transfer by employing only a single representation space. Furthermore, utilizing an additional representation space results in remarkable gains over zero-shot transfer. Given the flexibility of our framework, the results also suggest that \methodnew can deliver even better unsupervised transfer performance by taking advantage of new more powerful foundation models that will emerge in the future.

\section*{Acknowledgements}
We thank Chanakya Ekbote, Shuo Wen and Tingyang Yu for valuable suggestions that helped to improve the clarity of the manuscript. We also thank Nikita Doikov for fruitful discussions regarding efficient bilevel optimization techniques. We gratefully acknowledge the support of EPFL and ZEISS.
\section*{Impact Statement}
Although the main goal of our work is to advance the field of Machine Learning, the proposed framework relies on representation spaces of foundation models. These models inherit biases embedded in the data on which they were trained on \cite{bommasani2022opportunities}. Consequently, the extensive evaluation and alignment is recommended when deploying \methodnew to critical use-cases such as medicine.
\bibliography{references}
\bibliographystyle{icml2024}

\newpage
\appendix
\onecolumn
\section{Proof of Proposition \ref{hume_bound}}\label{appendix:proofs}
\renewcommand{\thetable}{\Alph{section}\arabic{table}}
\renewcommand\thefigure{\Alph{section}\arabic{figure}} 
\renewcommand\thealgorithm{\Alph{section}\arabic{algorithm}}
\renewcommand{\theHtable}{\Alph{section}\arabic{table}}
\renewcommand\theHfigure{\Alph{section}\arabic{figure}} 
\renewcommand\theHalgorithm{\Alph{section}\arabic{algorithm}}
\setcounter{table}{0}
\setcounter{figure}{0}
\setcounter{algorithm}{0}
Here, we first provide the simplified version of the main results from \citet{soudry2018implicit} for completeness and then present the proof of Proposition \ref{hume_bound}. For clarity, we overload notation for $x_n$ and consider $x_n$ is already represented in a representation space $\phi(x)$, \textit{i.e.}, $x_n = \phi(x_n)$.
Given binary labeling function $\tau(x) \in \{-1, +1\}$ of the dataset $\mathcal{D}=\{x_n\}_{n=1}^{N}$, let $\mathcal{L}(w)$ be the exponential loss function:
\begin{equation}\label{exp_loss}
    \mathcal{L}(w) = \sum_{n=1}^{N} \exp(-\tau(x_n) w^T x_n)
\end{equation}
\begin{assumption}\label{linearseparability}
    (\textit{Linear separability}) The dataset $\mathcal{D}$ is linearly separable: $\exists w_{*} \in \mathbb{R}^{d} \textrm{ s.t. } \tau(x_n) w^T_{*} x_n > 0$ for all $x_n \in \mathcal{D}$.
\end{assumption}
We consider minimizing (\ref{exp_loss}) using gradient descent with a step size $\eta$:
\begin{align}\label{exp_loss_gd}
    w_{m} = w_{m-1} - \eta \nabla_w \mathcal{L}(w_{m-1})
\end{align}
Let $w_{\text{SVM}}$ denote the \textit{primal} solution to the hard margin SVM problem:
\begin{align}\label{hard_margin_svm}
\begin{split}
    w_{\text{SVM}} = \min_{w} \quad & \|w\|_2^2 \\
    \textrm{s.t.} \quad & \tau(x_n) w^T x_n \geq 1 \quad \forall x_n \in \mathcal{D}.
\end{split}
\end{align}
Let $\alpha_{\text{SVM}}$ denote the \textit{dual} solution to the hard margin SVM problem:
\begin{align}\label{hard_margin_svm_dual}
\begin{split}
    \alpha_{\text{SVM}} = \max_{\alpha} \quad &  \sum_{n=1}^{N} \alpha_n - \frac{1}{2} \sum_{i=1}^{N} \sum_{j=1}^{N} \alpha_{i} \alpha_{j} \tau(x_i) \tau(x_j) x_i^T x_j \\
    \textrm{s.t.} \quad & \alpha_{n} \geq 0 \quad \forall n,
\end{split}
\end{align}
where primal and dual variables are related as $w_{\text{SVM}} = \sum_{n=1}^{N} (\alpha_{\text{SVM}})_n \tau(x_n) x_n$ \cite{vapnik1995nature}.
\begin{assumption}\label{supportvectorspan}
    (\textit{Non-degenerate dataset}) Support vectors $S=\{x_n \in \mathcal{D} | \tau(x_n) w_{\text{SVM}}^T x_n = 1\}$ span the data, \textit{i.e.}, $\textrm{rank}(\mathcal{D}_{S}) = \textrm{rank}(\mathcal{D})$, where $\mathcal{D}_{S}$ is a matrix whose columns are $x_n \in S$. Furthermore, for each $x_n \in S$, the corresponding dual variables are strictly positive, \textit{i.e.}, $(\alpha_{\text{SVM}})_n > 0$, and the rest are zero.
\end{assumption}
After above specifications, the simplified version of the seminal result by \citet{soudry2018implicit} is:
\begin{proposition}\label{app:implicitbiasgd}
    \textit{(Implicit Bias of Gradient Descent, \citet{soudry2018implicit})} For almost any non-degenerate (Assumption \ref{supportvectorspan}) dataset which is linearly separable (Assumption \ref{linearseparability}), any starting point $w_0$ and step size $\eta < 1 / \mathcal{L}(w_0)$, the gradient descent iterates (\ref{exp_loss_gd}) will behave as:
    \begin{equation}
        w_{m} = w_{\text{SVM}} \log m + \tilde{w} + r_{m},
    \end{equation}
    where $w_{\text{SVM}}$ is the max-margin vector (\ref{hard_margin_svm}), $\tilde{w}$ is a solution to:
    \begin{equation}\label{wtilde}
        \forall x_n \in S\ : \ \eta \exp(-\tau(x) \tilde{w}^T x_n) = (\alpha_{\text{SVM}})_n,
    \end{equation}
    and the residual $r_m$ is bounded with $\lim_{m \to \infty} \|r_m\|_2 = 0$
\end{proposition}
Equipped with this result, we analyze the generalization-based objective:
\begin{align}
        \mathcal{L}^{\text{binary}}_{M}(\theta) =  & \sum_{x \in \mathcal{D}} \exp(-\tau_\theta(x) w_{M}^T \phi(x)) \label{bin_general_outer_theta} \\
        & \textrm{s.t. } w_M = \Xi^{(M)}(w_0, \mathcal{D}), \label{bin_general_inner_theta}
\end{align}
where $\tau_\theta(x)=\sigma(\theta^T x)$ is the task encoder with an odd activation function such as $\tanh$ and $w_M = \Xi^{(M)}(w_0, \mathcal{D})$ denotes $M$ steps of gradient descent with step size $\eta$ and labeling defined by $\tau_\theta$. Without loss of generality, we can assume $\|x_n\|_2 \leq 1, \forall x_n \in \mathcal{D}$. Given the above specifications, we are now ready to state our main result:
\begin{proposition}\label{app:humesinglespacemargin}
    \textit{(Lower bound for the generalization-based objective)} Following the assumptions of Proposition \ref{app:implicitbiasgd}, given $M \gg 1$ and $\theta \neq 0$, we have:
    \begin{align}
        \mathcal{L}^{\text{binary}}_{M}(\theta) \geq g(\theta) \|w_{\text{SVM}}(\theta)\|_2^2,
    \end{align}
    where $g(\theta) = (M \eta \exp(\|r_M(\theta)\|_2))^{-1}$, the residual $r_M(\theta)$ is bounded with $\lim_{M \to \infty} \|r_M(\theta)\|_2 = 0$, and $w_{\text{SVM}}(\theta)$ is the solution of the hard-margin SVM for a given $\theta$:
    \begin{align}
    \begin{split}\label{hardmarginsvmfortheta}
        w_{\text{SVM}}(\theta) = \min_{w} \quad & \|w\|_2^2 \\
        \textrm{s.t.} \quad & \tau_{\theta}(x_n) w^T x_n \geq 1 \quad \forall x_n \in \mathcal{D}.
    \end{split}
    \end{align}
\end{proposition}
\begin{proof} 
The key observation is that the task encoder $\tau_\theta(x)$ generates linearly separable labelings, allowing us to apply Proposition \ref{app:implicitbiasgd} and substitute the explicit form of iterates $w_{M}$ into the outer objective (\ref{bin_general_outer_theta}). Indeed, for example, $w_{*}=\theta$ satisfies Assumption \ref{linearseparability}, \textit{i.e.}, $\tau_\theta(x_n) \theta^T x_n > 0$ for all $x_n \in \mathcal{D}$ and $\theta \neq 0$. Thus, substituting the iterates $w_{M}$ into the outer objective leads to:
\begin{equation}\label{binlosssubstitute}
    \mathcal{L}^{\text{binary}}_{M}(\theta) = \sum_{n=1}^{N} \exp(-\tau_\theta(x_n) \left(w_{\text{SVM}}(\theta) \log M + \tilde{w}(\theta) + r_M(\theta) \right)^T x_n ),
\end{equation}
where we explicitly indicate that $w_{\text{SVM}}(\theta)$, $\tilde{w}(\theta)$ and $r_{M}(\theta)$ depend on the parameters $\theta$ of the task encoder $\tau_\theta$. Let $\mathcal{L}_n(\theta)$ be $n$-th term of the sum and $S_{\theta}$ be the indices of support vectors, \textit{i.e.}, $n \in \{1, \dots, N\} \textrm{ s.t. } \tau_{\theta}(x_n) w_{\text{SVM}}^T x_n = 1$. Then, due to the non-negativity of $\exp(\cdot)$, we have: 
\begin{equation}\label{ineqexpnonnegativity}
    \mathcal{L}^{\text{binary}}_{M}(\theta) \geq \sum_{n \in S_\theta} \mathcal{L}_n (\theta).
\end{equation}
Considering single term $\mathcal{L}_{n}(\theta),\ n \in S_\theta$ and opening the brackets, we obtain:
\begin{equation}\label{loss_single_term}
    \mathcal{L}_{n}(\theta) = \underbrace{\exp(-\log M \cdot \tau_{\theta}(x_n) w_{\text{SVM}}^T(\theta) x_n)}_{\mathcal{L}_1} \underbrace{\exp(-\tau_{\theta}(x_n) \tilde{w}(\theta)^T x_n)}_{\mathcal{L}_2} \underbrace{\exp(-\tau_{\theta}(x_n) r_M(\theta)^T x_n)}_{\mathcal{L}_{3}}.
\end{equation}
Inspecting (\ref{loss_single_term}) separately for each term $\mathcal{L}_{i}$, we obtain: \textit{(i)} $\mathcal{L}_{1} = \frac{1}{M}$ since $n \in \mathcal{S}_{\theta}$; \textit{(ii)} $\mathcal{L}_{2} = \frac{(\alpha_{\text{SVM}}(\theta))_n}{\eta}$ by (\ref{wtilde}); and \textit{(iii)} $\mathcal{L}_{3} \geq \exp(-\|r_M(\theta)\|_2)$ by Cauchy–Schwarz inequality given that $\|\tau_\theta(x_n) x_n\|_2 \leq 1$. Combining this with (\ref{ineqexpnonnegativity}), finally we obtain:
\begin{equation}\label{finalineq}
    \mathcal{L}^{\text{binary}}_{M}(\theta) \geq (M \eta \exp(\|r_M(\theta)\|_2))^{-1} \sum_{n \in S_{\theta}} \alpha_n(\theta) = (M \eta \exp(\|r_M(\theta)\|_2))^{-1} \|w_{\text{SVM}}(\theta)\|_2^2,
\end{equation}
where the last equality comes from the fact that:
\begin{align*}
\begin{split}
\|w_{\text{SVM}}(\theta)\|_2^2 & = w_{\text{SVM}}(\theta)^T w_{\text{SVM}}(\theta) = w_{\text{SVM}}(\theta)^T \sum_{n \in S_{\theta}} (\alpha_{\text{SVM}}(\theta))_n \tau_{\theta}(x_n) x_n = \\ & = \sum_{n \in S_{\theta}} (\alpha_{\text{SVM}}(\theta))_n \tau_{\theta}(x_n) w_{\text{SVM}}(\theta)^T x_n =  \sum_{n \in S_{\theta}} (\alpha_{\text{SVM}}(\theta))_n \cdot 1 = \sum_{n \in S_{\theta}} (\alpha_{\text{SVM}}(\theta))_n,
\end{split}
\end{align*}
concluding the proof.
\end{proof}

\newpage
\section{Experimental Details}\label{appendix:expdetails}
\renewcommand{\thetable}{\Alph{section}\arabic{table}}
\renewcommand\thefigure{\Alph{section}\arabic{figure}} 
\renewcommand\thealgorithm{\Alph{section}\arabic{algorithm}}
\renewcommand{\theHtable}{\Alph{section}\arabic{table}}
\renewcommand\theHfigure{\Alph{section}\arabic{figure}} 
\renewcommand\theHalgorithm{\Alph{section}\arabic{algorithm}}
\setcounter{table}{0}
\setcounter{figure}{0}
\setcounter{algorithm}{0}

\subsection{Datasets}\label{app:datasets}
We evaluate our framework on 26 vision datasets studied in~\citet{radford2021learning}.
These datasets cover a wide range of vision tasks, including general object classification datasets CIFAR10~\cite{krizhevsky2009learning}, CIFAR100~\cite{krizhevsky2009learning}, STL10~\cite{coates2011analysis}, ImageNet~\cite{deng2009imagenet}, Caltech101~\cite{fei2004learning}; fine-grained object classification datasets Food101~\cite{bossard2014food}, Flowers~\cite{nilsback2008automated}, Birdsnap~\cite{berg2014birdsnap}, Stanford Cars~\cite{krause20133d}, FGVC Aircraft~\cite{maji2013fine}, Oxford Pets~\cite{parkhi2012cats}; handwritten digits classification dataset MNIST~\cite{lecun1998gradient}; texture classification dataset DTD~\cite{cimpoi2014describing}; scene classification dataset SUN397~\cite{xiao2016sun}; the facial emotion recognition dataset FER2013~\cite{goodfellow2013challenges}; the satellite image
classification datasets EuroSAT~\cite{helber2019eurosat}, RESISC45~\cite{cheng2017remote}; the German Traffic Sign Recognition Benchmark (GTSRB)~\cite{stallkamp2012man}; the KITTI Distance dataset~\cite{geiger2012we}; the metastatic tissue classification dataset PatchCamelyon (PCam)~\cite{veeling2018rotation}; action recognition datasets UCF101~\cite{soomro2012ucf101}, Kinetics700~\cite{carreira2019short}; the CLEVR counting dataset~\cite{johnson2017clevr}; the Hateful Memes dataset~\cite{kiela2020hateful}; the country classification dataset Country211~\cite{radford2021learning} and the Rendered SST2 dataset~\cite{radford2021learning}. For CLEVR, we take 2000 random samples as training split and 500 random samples as test split. For two video datasets UCF101 and Kinetics700, we take the middle frame of each video clip as the input of the pre-trained models. Details of each dataset are provided in Table~\ref{tab:dataset}. We use \textit{accuracy} as the evaluation metric for all the datasets. Finally, it's worth noting that \methodnew could also be applied to the tasks in various modalities besides vision or even in cross-modalities scenarios, provided that the pre-trained representations are available.
\\

\begin{table*}[ht]
\centering
\renewcommand{\arraystretch}{1.1}
\caption{\xhdr{Benchmark suite of 26 datasets} We use \textit{accuracy} as the evaluation metric for all datasets.}
\vspace{6pt}
\begin{tabular}{lccc}
\Xhline{2\arrayrulewidth}
 Dataset & Number of Classes & Train size & Test size\\
 \hline
 Food101~\cite{bossard2014food} & 101& 75,750& 25,250\\
 CIFAR10~\cite{krizhevsky2009learning} & 10 & 50,000& 10,000\\
 CIFAR100~\cite{krizhevsky2009learning} & 100 & 50,000& 10,000\\
 Birdsnap~\cite{berg2014birdsnap} & 500 & 37,221& 2,500\\
 SUN397~\cite{xiao2016sun} & 397& 19,850& 19,850\\
 StanfordCars~\cite{krause20133d} & 196& 8,144& 8,041\\
 FGVC Aircraft~\cite{maji2013fine} & 100& 6,667& 3,333\\
 DTD~\cite{cimpoi2014describing} & 47 & 3,760& 1,880\\
 OxfordPets~\cite{parkhi2012cats} & 37 & 3,680& 3,669\\
 Caltech101~\cite{fei2004learning} & 102 & 3,060& 6,084\\
 Flowers~\cite{nilsback2008automated} &102 &  2,040& 6,149\\
 MNIST~\cite{lecun1998gradient} & 10 & 60,000& 10,000\\
 FER2013~\cite{goodfellow2013challenges} &7 &  28,709& 3,589\\
 STL10~\cite{coates2011analysis} &10 &  5,000& 8,000\\
 EuroSAT~\cite{helber2019eurosat} &10 &  10,000& 5,000\\
 RESISC45~\cite{cheng2017remote} &45 &  25,200& 6,300\\
 GTSRB~\cite{stallkamp2012man}&43 &  26,640& 12,630\\
 KITTI Distance~\cite{geiger2012we}& 4 & 5,985& 1,496\\
 Country211~\cite{radford2021learning} & 211 & 42,200& 21,100\\ 
 PatchCamelyon~\cite{veeling2018rotation}&2 &  294,912& 32,768\\
 UCF101~\cite{soomro2012ucf101} &101 &  9,537& 3,783\\ 
 Kinetics700~\cite{carreira2019short} &700 &  536,485& 33,966\\
 CLEVR Counts~\cite{johnson2017clevr} &8 &  2,000& 500\\
 HatefulMemes~\cite{kiela2020hateful} & 2 & 8,500& 500\\
 The Rendered SST2~\cite{radford2021learning}& 2 & 7,792& 1,821\\
 ImageNet~\cite{deng2009imagenet} & 1000 & 1,281,167& 50,000\\
\Xhline{2\arrayrulewidth}
\end{tabular}
\label{tab:dataset}
\end{table*}

\newpage
\subsection{Representations}\label{app:representations}
\methodnew is compatible with any pre-trained representations. 
This paper presents the comprehensive evaluation of \methodnew on a wide range of representation spaces that vary on the pre-training datasets, model architectures and training objectives. Specifically, we consider CLIP ResNets (RN50, RN101, RN50x4, RN50x16, RN50x64) and CLIP Vision Transformers (ViT-B/32, ViT-B/16, ViT-L/14) pre-trained on WebImageText-400M~\cite{radford2021learning} for training \methodnew 1-space.
These models are pre-trained on the same dataset, scaling with number of the parameters.
For \methodnew 2-spaces, we incorporate DINOv2 ViT-g/14 pre-trained on LVD-142M~\cite{oquab2023dinov2} as the second representation space. 
Moreover, we also include SWAG ViT-H/14 pre-trained on IG-3.6B~\cite{singh2022revisiting}, CoCa ViT-L/14~\cite{yu2022coca} pre-trained on LAION-2B~\cite{schuhmann2022laionb}~\footnote{Since the original paper does not release the models, we use the reproduced version from the OpenCLIP project, which could be found at~\href{https://github.com/mlfoundations/open_clip}{https://github.com/mlfoundations/open\_clip}.}, OpenCLIP ViT-L/14 pre-trained on LAION-2B~\cite{cherti2023reproducible}, MOCOv3 ViT-B/16 pre-trained on ImageNet-1000~\cite{chen2021empirical} and EsViT Swim-B pre-trained on ImageNet-1000~\cite{li2022efficient} to study whether incorporating stronger representations on the given dataset may lead to the increased performance of unsupervised transfer with \methodnew. For all the models, we precompute the representations with standard image preprocessing pipelines and do not use any data augmentations during training of \methodnew. The details of pre-trained representations are provided on Table~\ref{tab:representations}.
\\
\\

\begin{table*}[ht]
\centering
\renewcommand{\arraystretch}{1.15}
\caption{\xhdr{Representation spaces used in \methodnew} ``Weakly Supervised" indicates that the model is pre-trained with text supervision, such as caption or hashtag of the image.}
\vspace{8pt}
\begin{tabular}{ccccc}
\Xhline{2\arrayrulewidth}
    Model & Architecture & Parameters & Trained on & Weakly Supervised \\
\hline
    \multirowcell{8}{CLIP\\~\cite{radford2021learning}} 
     & RN50 & 100M & \multirowcell{8}{WebImageText-400M} & \multirowcell{8}{\cmark} \\
     & RN101 & 120M &  \\
     & RN50x4 & 180M &  \\
     & RN50x16 & 290M &  \\
     & RN50x64 & 620M &  \\
     & ViT-B/32 & 150M & \\
     & ViT-B/16 & 150M &  \\
     & ViT-L/14 & 430M &  \\ \cmidrule{2-5}
    OpenCLIP~\cite{cherti2023reproducible} & ViT-L/14 & 430M &LAION-2B &\cmark \\
    SWAG~\cite{singh2022revisiting} & ViT-H/14 & 630M & IG-3.6B & \cmark \\
    CoCa~\cite{yu2022coca} & ViT-L/14 & 640M & LAION-2B & \cmark \\
    MOCOv3~\cite{chen2021empirical} & ViT-B/16 & 86M &ImageNet-1000 & \xmark \\
    EsViT~\cite{li2022efficient} & Swin-B & 88M & ImageNet-1000 & \xmark \\
    DINOv2~\cite{oquab2023dinov2} & ViT-g/14 & 1.1B & LVD-142M & \xmark \\
 \Xhline{2\arrayrulewidth}   
\end{tabular}
\label{tab:representations}
\end{table*}

\newpage
\subsection{Implementation Details}\label{app:implementationdetails}
\xhdr{Efficient alternating optimization} \methodnew contains a bilevel objective that measures the loss of the task encoder using the training loss of a linear classifier trained on the task produced by the task encoder. The hyper-gradient of the task encoder is $\nabla_\theta \mathcal{L} = \frac{\partial \mathcal{L}}{\partial \theta} + (\frac{\partial w}{\partial \theta})^T \frac{\partial \mathcal{L}}{\partial w}|_{w=w^*}$, where the Jacobian $\frac{\partial w}{\partial \theta}$ is generally expensive to obtain. Existing works usually estimate the hyper-gradient via unrolling or approximation based on the implicit function theorem, e.g., see~\citet{finn2017model, lorraine2020optimizing, ji2021bilevel, kwon2023fully, dagreou2022framework, bolte2023one, liu2022bome}. 
However, these methods might be inefficient and suboptimal in practice~\cite{scieur2022curse}. 
Fortunately, in the \methodnew framework, one could avoid the estimation of $(\frac{\partial w}{\partial \theta})^T \frac{\partial \mathcal{L}}{\partial w}$ given the fact that $\frac{\partial \mathcal{L}}{\partial w}|_{w=w^*}\approx0$. Thus, the gradient of the task encoder is simplified to $\nabla_\theta \mathcal{L}=\frac{\partial \mathcal{L}}{\partial \theta}|_{w=w^*}$. This inspires us to train the task encoder via \textit{alternating optimization}, which has been shown efficient for the min-min optimization problems~\cite{ablin2020super}. At each iteration, we first fix the task encoder and train the linear classifier for $M$ steps to find its approximate optima. Note that one could choose to re-initialize the linear classifier every time (\textit{cold-start}), or just start from the values of last iteration (\textit{warm-start}), which might introduce different implicit bias as noted by \citet{vicol2022implicit}.
After that, we update the task encoder based on the loss of the linear classifier. The training is efficient since no second-order gradient is needed in this process.
The pseudo-code of \methodnew is provided in Algorithm~\ref{alg:turtle}.

\begin{algorithm}
\caption{\methodnew for Unsupervised Transfer}
    \begin{algorithmic}[1]
        \STATE \textbf{Input}: Dataset $\mathcal{D}$, number of classes $C$, number of iterations $T$, representation spaces $\phi_1(\cdot), ..., \phi_K(\cdot)$, task parameters $\theta=\{\theta_1,...,\theta_K\}$, linear classifiers $w^1,...,w^K$, learning rate $\eta$, optimization operator $\Xi(\cdot)$, number of adaptation steps $M$, entropy regularization weight $\gamma$
        {\color{OliveGreen} \hfill // $\Xi(\cdot)$ can be any iterative operator, e.g., gradient descent}
        \STATE Randomly initialize $\theta_1, ..., \theta_K$ and $w^1_0, ...,w^K_0$
        \FOR{$t=1$ to $T$}
            \STATE Sample mini-batch from dataset $X\sim \mathcal{D}$
            \STATE Generate task from task encoder $\tau_\theta \leftarrow \frac{1}{K}\sum_{k=1}^K{\sigma(\theta_k\phi_k(X))}$
            \STATE Update linear classifiers for $M$ steps $\forall k \in [K]: w^k_M\leftarrow \Xi^{(M)}(w^k_0, X)$ \vskip 3pt 
            \STATE Update task parameters $\forall k \in [K]: \theta_k \leftarrow \theta_k - \eta \frac{\partial }{\partial \theta_k} \left[\mathcal{L}_{ce}({w^k_M}\phi_k(X), \tau_\theta) + \gamma \mathcal{R}(\overline{\tau}_\theta)\right]$ {\color{OliveGreen} \hfill // partial derivative $\frac{\partial}{\partial \theta_k}$} \vskip 2pt
            \STATE \algorithmicif\ warm-start \algorithmicthen \ update start points $\forall k\in [K], w^k_0 \leftarrow w^k_M$
            {\color{OliveGreen} \hfill // cold-start keeps the initial $w^k_0$}

        \ENDFOR

        \STATE \textbf{Output}: Task parameters $\theta=\{\theta_1,...,\theta_K\}$
    \end{algorithmic}
\label{alg:turtle}
\end{algorithm}

\xhdr{Training details}
We precompute the feature representations for all datasets before the training.
We use Weight Normalization~\cite{salimans2016weight} to parameterize the task encoder since we found it helpful for the convergence. ADAM~\cite{kingma2014adam} optimizer is used for the training of both linear classifier and task encoder. We use $10000$ as the default batch-size. For datasets smaller than $10000$, we train the model with full-batch at each iteration. Overall, we found \methodnew is robust to the choice of the batch-size. We update the linear classifier for $M=10$ steps at each iteration and train the task encoder for $T=6000$ iterations in total. If not specifically mentioned, we set the entropy regularization parameter $\gamma=10$ for all experiments. We show robustness of \methodnew to this hyperparameter in Appendix \ref{app:entropyablation}. For each dataset, we do a grid search over $5$ different learning rates for both task encoder and linear classifier with $\eta \in \{0.01, 0.005, 0.001, 0.0005, 0.0001\}$, respectively. We combine each pair of learning rates with the choice of warm-start or cold-start, and finally get set of $50$ triplets to search over for each dataset. We use \textit{cross-validation} to select hyper-parameters, as described below. Following~\citet{gadetsky2023pursuit, van2020scan}, we use Hungarian algorithm~\cite{kuhn1955hungarian} to match the labeling found by \methodnew and the ground truth labeling to compute the clustering accuracy. If not specified, we train our model on the training split and report the clustering accuracy on the test split. In Section~\ref{app:testtestexp}, we also consider the setting of both training and evaluating \methodnew on the test split, mimicking low data regime. 

\xhdr{Cross-validation for task selection} For each dataset, we obtain $50$ tasks after grid search, \textit{i.e.}, each corresponds to the hyperparameter triplet. We use $10$-fold cross-validation to select the best task. The cross-validation regards the learned task as ``pseudo-labels'' and measures the \textit{generalization error} of a linear classifier trained on these ``pseudo-labels''. Specifically, we randomly split the dataset into $10$ folds. In each round, a linear classifier is trained on $9$ folds and tested on the rest fold. 
The final cross-validation score is the average test accuracy over all rounds. Importantly, this process relies solely on the learned tasks and does \textit{not} need any information about the ground-truth labels. For \methodnew trained on multiple representations, we do cross-validation on each representation space separately and average the final scores. 
The task with the highest cross-validation score is selected as the final output of \methodnew. 
Figure~\ref{fig:cross_val} shows the performance of the learned tasks obtained by \methodnew 2-spaces CLIP ViT-L/14 and DINOv2 and their corresponding cross-validation scores over 26 datasets. As indicated by the plot, the cross-validation score is well correlated with the clustering accuracy with an average of $\rho=0.61$ two-sided Pearson correlation coefficient over $26$ datasets. 
Moreover, among $20$ datasets, cross-validation successfully identifies the best or near-best task (\textit{i.e.}, with less than $1.5$ point difference of clustering accuracy). The result of cross-validation also empirically verifies the effectiveness of the generalization-based objective and suggests that the labelings with low generalization error tend to be more aligned with human labeled tasks, confirming the original findings of \citet{gadetsky2023pursuit} on the wide suite of datasets.

\begin{figure}[t!]
    \centering
    \includegraphics[width=0.95\linewidth]{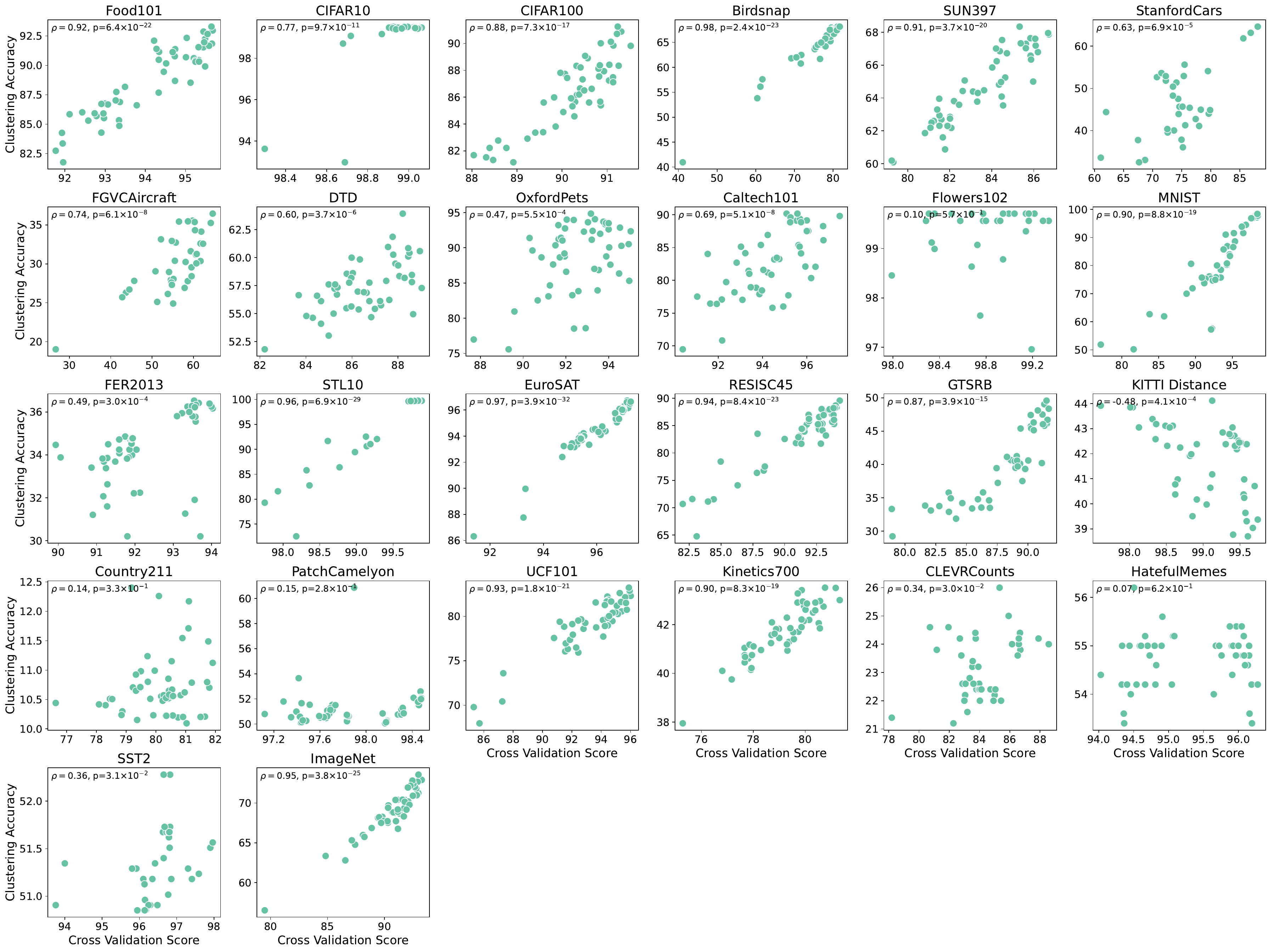}
    \caption{\textbf{Task selection via cross-validation}. We use \methodnew 2-spaces CLIP ViT-L/14 and DINOv2 to produce the tasks. We show the cross-validation score and corresponding clustering accuracy of the tasks learned by \methodnew with $50$ different hyperparameters for each dataset. The cross-validation score is well correlated with the clustering accuracy ($\rho=0.61$ of two-sided Pearson correlation coefficient averaged over 26 datasets).}
    \label{fig:cross_val}
    \vspace{-1em}
\end{figure}

\xhdr{Linear probe} Supervised linear probe is a widely used method to evaluate the quality of  representation learning~\cite{radford2021learning, oquab2023dinov2}. It trains a linear classifier on the train split on top of the representations extracted from the pre-trained models and then evaluates the performance on the test split. We use the \texttt{cuML.LogisticRegression}~\cite{raschka2020machine} for linear probe evaluation in our paper~\footnote{This library is available at ~\href{https://github.com/rapidsai/cuml}{https://github.com/rapidsai/cuml}.}. The linear classifier is trained with L-BFGS optimizer for maximum of $1000$ iterations. The \texttt{cuML} library allows for GPU acceleration and, thus, it is much faster than \texttt{sklearn.linear\_model.LogisticRegression} counterpart, especially on large datasets such as ImageNet. To determine the strength of L2 norm regularization, we randomly take 20\% of the training split for validation and search over $96$ values in the log-space ranging from $10^{-6}$ to $10^{6}$. The selection process takes a few minutes on small datasets, and around $8$ hours on ImageNet, with a single NVIDIA A100 GPU. After that, we train the model with the best regularization strength on the entire training split and report the classification accuracy on the test split.

\section{Details on Unsupervised Baselines and Numerical Results}\label{appendix:unsuperviseddetails}
\renewcommand{\thetable}{\Alph{section}\arabic{table}}
\renewcommand\thefigure{\Alph{section}\arabic{figure}} 
\renewcommand\thealgorithm{\Alph{section}\arabic{algorithm}}
\renewcommand{\theHtable}{\Alph{section}\arabic{table}}
\renewcommand\theHfigure{\Alph{section}\arabic{figure}} 
\renewcommand\theHalgorithm{\Alph{section}\arabic{algorithm}}
\setcounter{table}{0}
\setcounter{figure}{0}
\setcounter{algorithm}{0}

\xhdr{K-Means clustering} We apply K-Means \cite{macqueen1967some} clustering on top of pre-trained features as a simple baseline that does not require task-specific representation learning. Similarly to the linear probe, we also use the implementation from \texttt{CuML} library for the GPU acceleration (i.e., \texttt{CuML.KMeans}). For each dataset and the corresponding representation, we train K-Means with maximum $1000$ iterations (\texttt{max\textunderscore iter=1000}) and $10$ random initializations (\texttt{n\textunderscore init=10}) on the train split, and report the clustering accuracy on the test split. In the case when multiple representations are used, we first L2 normalize representation from each pre-trained model, and then apply K-Means clustering on top of the concatenation of all L2 normalized features.

\xhdr{\methodold} \methodold \cite{gadetsky2023pursuit} is the recent state-of-the-art unsupervised learning baseline that introduced the instantiation of the generalization-based objective (\ref{hume_instantiation}). Specifically, it learns task-specific representations on the target dataset to model the task encoder, and then measures the generalization error of a linear classifier in the representation space of a foundation model. We use the original source code~\footnote{Code could be found at ~\href{https://github.com/mlbio-epfl/hume}{https://github.com/mlbio-epfl/hume}.} for the implementation of \methodold, with modifications to improve the speed and performance.
In particular, we replace task-specific representations with a pre-trained foundation model since we empirically found it yields better performance.  Besides, we remove the variance reduction used in the original \methodold and sample only the single mini-batch at every iteration (\textit{i.e.}, the same as \methodnew), since we found it significantly reduces the computational cost and does not influence the final performance.
We update the linear classifier with $M=300$ steps at each iteration, and train the task encoder with $T=6000$ iterations in total. The default batch-size is set to $10000$.
Moreover, we follow the same hyperparameter selection procedure of \methodnew to select the inner/outer learning rates and warm-start/cold-start for \methodold.

\xhdr{Comparison of \methodnew to \methodold and K-Means} For a fair comparison, we train \methodnew, \methodold and K-Means using the same representation spaces, i.e., CLIP ViT-L/14 and DINOv2 ViT-g/14. Given that \methodold's task encoder parametrization uses only the single space, we run \methodold with the task encoder modeled using CLIP or DINOv2 (denoted as \methodold \textit{CLIP} and \methodold \textit{DINOv2} respectively), and measure the generalization error using the rest representation space.
For each method, we report the training time in minutes and the clustering accuracy averaged over $3$ random seeds. For each random seed, we perform the hyperparameter selection for \methodold and \methodnew as described in the corresponding subsection above. Table~\ref{tab:unsupervised_acc} and Table~\ref{tab:unsupervised_time} show the obtained results on $5$ datasets. Overall, the results indicate that \methodnew outperforms \methodold and K-Means on all the considered datasets, highlighting the effectiveness of design choices made in \methodnew. For instance, combining multiple representation spaces for modeling the task encoder in \methodnew brings substantial gains compared to \methodold. Namely, \methodnew achieves remarkable 28\% and 23\% absolute improvement (40\% and 30\% relative improvement) over \methodold \textit{DINOv2} and \methodold \textit{CLIP} respectively on the MNIST dataset. Furthermore, efficient first-order optimization techniques used in \methodnew allow for fast optimization, taking just $5$ minutes even on large-scale datasets such as ImageNet.

\begin{table*}[ht]
\centering
\vspace*{-2mm}
\renewcommand{\arraystretch}{1.0}
\caption{\xhdr{Accuracy of \methodnew and unsupervised baselines} The results are averaged with standard deviations computed over $3$ runs.}
\vspace{6pt}
\begin{tabular}{l|ccccc}
\midrule[1pt]
{\bf Method} 
& \bf MNIST & \bf CIFAR100 & \bf Food101 & \bf Birdsnap & \bf ImageNet \\
\midrule
K-Means & 68.9 {\footnotesize $\pm$ 0.6}  
        & 75.1 {\footnotesize $\pm$ 0.5}  
        & 78.0 {\footnotesize $\pm$ 0.7} 
        & 54.0 {\footnotesize $\pm$ 0.8}  
        & 64.8 {\footnotesize $\pm$ 0.3} \\
\methodold \textit{CLIP} 
        & 75.3 {\footnotesize $\pm$ 5.0} 
        & 71.2 {\footnotesize $\pm$ 2.2} 
        & 86.5 {\footnotesize $\pm$ 1.3}  
        & 45.8 {\footnotesize $\pm$ 1.8}  
        & 65.2 {\footnotesize $\pm$ 0.9} \\
\methodold \textit{DINOv2} 
        & 69.7 {\footnotesize $\pm$ 5.9}  
        & 83.9 {\footnotesize $\pm$ 1.2}  
        & 85.3 {\footnotesize $\pm$ 1.7} 
        & 57.3 {\footnotesize $\pm$ 0.7} 
        & 68.1 {\footnotesize $\pm$ 0.2} \\
\methodnew  & \bf98.0 {\footnotesize $\pm$ 0.4}  
        & \bf89.1 {\footnotesize $\pm$ 1.0}  
        & \bf92.8 {\footnotesize $\pm$ 0.5}  
        & \bf67.8 {\footnotesize $\pm$ 0.4}  
        & \bf72.4 {\footnotesize $\pm$ 0.3}  \\
\midrule[1pt]
\end{tabular}
\label{tab:unsupervised_acc}
\end{table*}

\begin{table*}[ht]
\centering
\vspace*{-5mm}
\renewcommand{\arraystretch}{1.0}
\caption{\xhdr{Training time (in minutes) of \methodnew and unsupervised baselines} The results are averaged with standard deviations computed over $3$ runs. The standard deviation for K-Means and \methodnew is negligible.}
\vspace{6pt}
\begin{tabular}{l|ccccc}
\midrule[1pt]
{\bf Method} 
& \bf MNIST & \bf CIFAR100 & \bf Food101 & \bf Birdsnap & \bf ImageNet \\
\midrule
K-Means & 0.02 {\footnotesize $\pm$ 0.0}
        & 0.04 {\footnotesize $\pm$ 0.0}
        & 0.1 {\footnotesize $\pm$ 0.0}
        & 0.1 {\footnotesize $\pm$ 0.0}
        & 6.4 {\footnotesize $\pm$ 0.0} \\
\methodold \textit{CLIP} 
        & 15.2 {\footnotesize $\pm$ 2.4}
        & 44.1 {\footnotesize $\pm$ 0.5}
        & 45.5 {\footnotesize $\pm$ 1.3}
        & 156.8 {\footnotesize $\pm$ 1.2}
        & 285.4 {\footnotesize $\pm$ 4.1} \\
\methodold \textit{DINOv2} 
        & 11.1 {\footnotesize $\pm$ 0.3}
        & 31.7 {\footnotesize $\pm$ 0.4}
        & 31.0 {\footnotesize $\pm$ 0.3}
        & 96.7 {\footnotesize $\pm$ 0.6}
        & 185.6 {\footnotesize $\pm$ 7.9} \\
\methodnew  & 1.6 {\footnotesize $\pm$ 0.0}
        & 1.6 {\footnotesize $\pm$ 0.0}
        & 1.7 {\footnotesize $\pm$ 0.0}
        & 2.1 {\footnotesize $\pm$ 0.0}
        & 4.8 {\footnotesize $\pm$ 0.0}
\\

\midrule[1pt]
\end{tabular}
\label{tab:unsupervised_time}
\end{table*}

\newpage
\clearpage
\section{Complete List of Individual Numerical Results}\label{appendix:allnumericresults}
\renewcommand{\thetable}{\Alph{section}\arabic{table}}
\renewcommand\thefigure{\Alph{section}\arabic{figure}} 
\renewcommand\thealgorithm{\Alph{section}\arabic{algorithm}}
\renewcommand{\theHtable}{\Alph{section}\arabic{table}}
\renewcommand\theHfigure{\Alph{section}\arabic{figure}} 
\renewcommand\theHalgorithm{\Alph{section}\arabic{algorithm}}
\setcounter{table}{0}
\setcounter{figure}{0}
\setcounter{algorithm}{0}

\vspace{5pt}
\begin{table*}[ht!]
\renewcommand{\arraystretch}{1.17}
\vspace{-2em}
\linespread{1}
\aboverulesep = 0.2em \belowrulesep = 0.2em
\scriptsize
\centering
\caption{\xhdr{Complete list of numerical results} Results of supervised linear probe, CLIP zero-shot transfer, K-Means clustering and \methodnew unsupervised transfer on $26$ datasets and $9$ representation spaces (CLIP RN50, RN101, RN50x4, RN50x16, RN50x64, ViT-B/32, ViT-B/16, ViT-L/14 and DINOv2 ViT-g/14). The results for K-Means 2-spaces and \methodnew 2-spaces are obtained using DINOv2 and the corresponding CLIP model.}
\vspace{3pt}
\begin{tabular}{ll|cccccccccccccccccccccccccccc}
    & &\hspace{-1.2em} & \rotatebox[origin=lb]{90}{\smash{Food101}} & \rotatebox[origin=lb]{90}{\smash{CIFAR10}} & \rotatebox[origin=lb]{90}{\smash{CIFAR100}} & \rotatebox[origin=lb]{90}{\smash{Birdsnap}} & \rotatebox[origin=lb]{90}{\smash{SUN397}} & \rotatebox[origin=lb]{90}{\smash{Stanford Cars}} & \rotatebox[origin=lb]{90}{\smash{FGVC Aircraft}} & \rotatebox[origin=lb]{90}{\smash{DTD}} & \rotatebox[origin=lb]{90}{\smash{Oxford Pets}} & \rotatebox[origin=lb]{90}{\smash{Caltech101}} & \rotatebox[origin=lb]{90}{\smash{Flowers102}} & \rotatebox[origin=lb]{90}{\smash{MNIST}} & \rotatebox[origin=lb]{90}{\smash{FER2013}} & \rotatebox[origin=lb]{90}{\smash{STL10}} & \rotatebox[origin=lb]{90}{\smash{EuroSAT}} & \rotatebox[origin=lb]{90}{\smash{RESISC45}} & \rotatebox[origin=lb]{90}{\smash{GTSRB}} & \rotatebox[origin=lb]{90}{\smash{KITTI}} & \rotatebox[origin=lb]{90}{\smash{Country211}} & \rotatebox[origin=lb]{90}{\smash{PCam}} & \rotatebox[origin=lb]{90}{\smash{UCF101}} & \rotatebox[origin=lb]{90}{\smash{Kinetics700}} & \rotatebox[origin=lb]{90}{\smash{CLEVR}} & \rotatebox[origin=lb]{90}{\smash{HatefulMemes}} & \rotatebox[origin=lb]{90}{\smash{Rendered SST2}} & \rotatebox[origin=lb]{90}{\smash{ImageNet}} & \rotatebox[origin=lb]{90}{\smash{Average}} \\
\bottomrule
    \end{tabular}
    \label{tab:big-big-scaling-table}
\vspace{-8em}
\end{table*}

\newpage
\section{Additional Results on 26 Vision Datasets}
\label{appendix:additional_results}
\renewcommand{\thetable}{\Alph{section}\arabic{table}}
\renewcommand\thefigure{\Alph{section}\arabic{figure}} 
\renewcommand\thealgorithm{\Alph{section}\arabic{algorithm}}
\renewcommand{\theHtable}{\Alph{section}\arabic{table}}
\renewcommand\theHfigure{\Alph{section}\arabic{figure}} 
\renewcommand\theHalgorithm{\Alph{section}\arabic{algorithm}}
\setcounter{table}{0}
\setcounter{figure}{0}
\setcounter{algorithm}{0}
We show all experimental results of supervised transfer with linear probe, CLIP zero-shot transfer, and \methodnew unsupervised transfer on 26 vision datasets in Table~\ref{tab:big-big-scaling-table}. The linear probe performance of DINOv2 ViT-g/14 is also included for reference. As indicated in the table, \methodnew achieves strong unsupervised transfer performance across various datasets and models. 
For example, as illustrated in Figure~\ref{fig:turtle_clip_zs}, \methodnew 1-space CLIP ViT-L/14 surpasses the corresponding CLIP zero-shot transfer on 13 out of 26 datasets. When trained with multiple representations (\textit{i.e.}, using CLIP models and DINOv2 ViT-g/14), \methodnew 2-spaces achieves superior performance on most datasets compared to \methodnew 1-space. Remarkably, on datasets such as MNIST and CIFAR100, the absolute improvement is 31\% and 21\%, indicating the effectiveness of \methodnew in combining the knowledge of multiple foundation models. Furthermore, as shown in Figure~\ref{fig:turtle_clipdino_zs}, \methodnew trained with CLIP ViT-L/14 and DINOv2 ViT-g/14 outperforms CLIP zero-shot on 15 out of 26 datasets.

In addition, we compare the performance of \methodnew to supervised linear probe using single representation space in Figure~\ref{turtle_linear_datasets}. It can be seen that there exists a strong positive correlation between the performance of unsupervised transfer and supervised linear probe. Figure~\ref{turtle_linear_datasets_multiple} provides the analysis of \methodnew 2-spaces trained using CLIP ViT-L/14 and DINOv2 ViT-g/14, indicating that the performance of \methodnew 2-spaces is also strongly correlated with the average linear probe performance. Overall, these results suggest that \methodnew could potentially benefit from the improved quality of representations, as measured by supervised linear probe.

Finally, it's worth noting that \methodnew 2-spaces might underperform \methodnew 1-space on some datasets, as shown in Table~\ref{tab:big-big-scaling-table}. We hypothesize that the discrepancy might stem from the suboptimality of DINOv2 representations for the tasks heavily related to semantic comprehension, such as semantic analysis (Rendered SST, HatefulMemes), traffic sign recognition (GTSRB), geolocation (Country211) and object counting (CLEVR). Since DINOv2 is pre-trained with self-supervised objective~\cite{oquab2023dinov2}, the learned features might not be directly transferable to these semantic-intensive tasks. Such trend could also be observed by the linear probe performance, where CLIP ViT-L/14 outperforms DINOv2 ViT-g/14 by a large margin on the Rendered SST2, CLEVR, Country211, GTSRB and FER2013. Therefore, incorporating DINOv2 representations might not yield optimal results for these specific datasets.

\begin{figure}[ht]
\begin{minipage}[t]{0.45\textwidth}
\centering
\includegraphics[width = 0.82\linewidth]{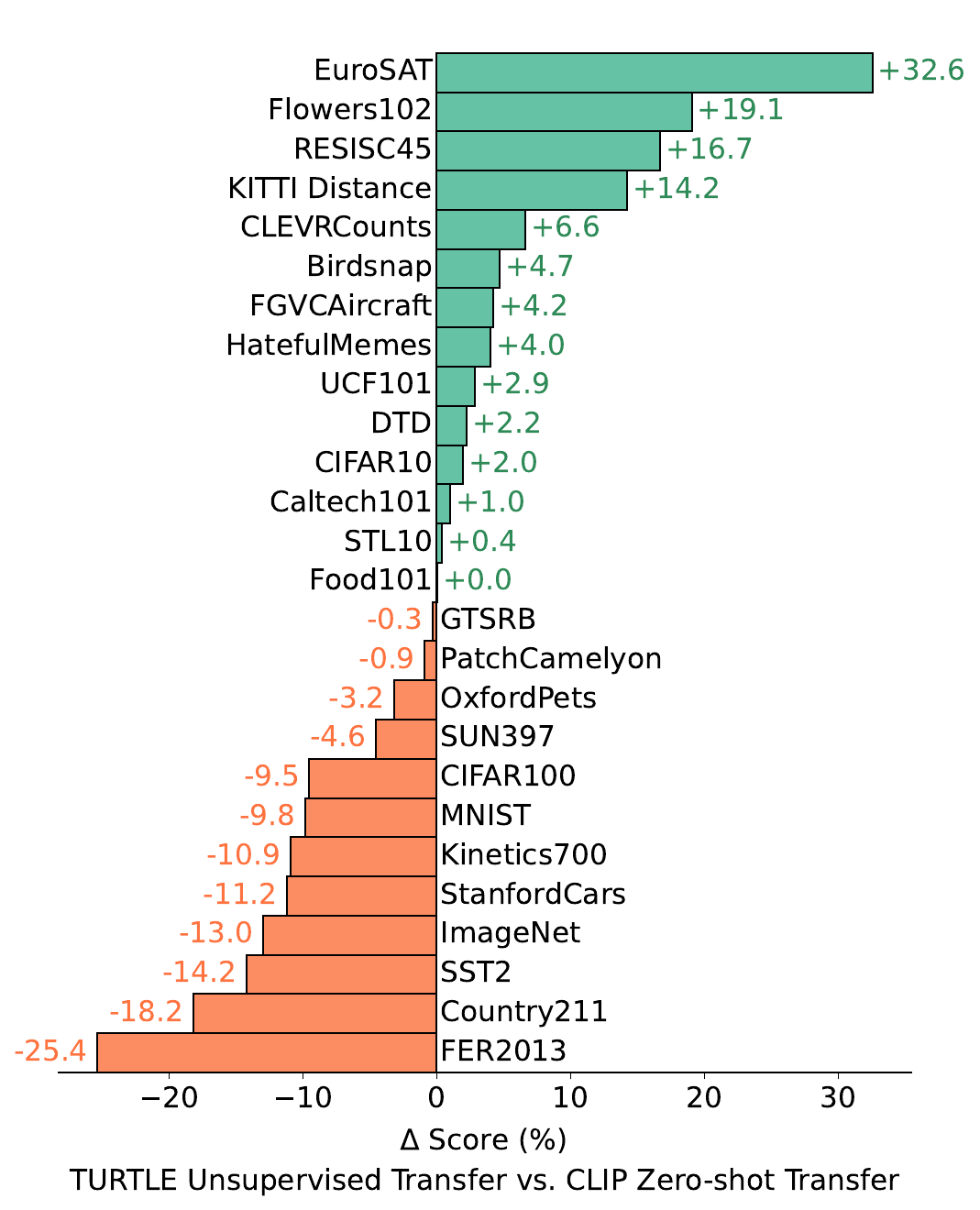}
\caption{\textbf{Using the same representation space, \methodnew outperforms CLIP zero-shot classifier on 13 out of 26 datasets.
} \methodnew is trained with CLIP ViT-L/14 and does not require any supervision. CLIP zero-shot utilizes the same architecture, but requires the additional text encoder and description of visual categories.} \label{fig:turtle_clip_zs}
\end{minipage}
\begin{minipage}[t]{0.05\textwidth}
~~~
\end{minipage}
\begin{minipage}[t]{0.48\textwidth}
\centering
\includegraphics[width = 0.98\linewidth]{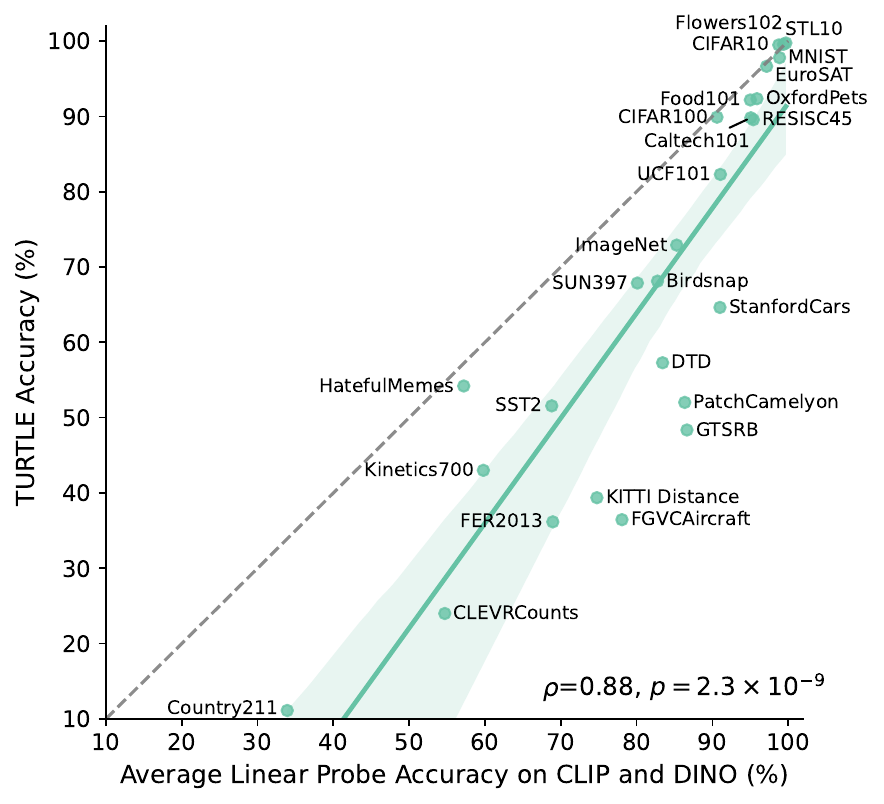}
\caption{\xhdr{Unsupervised transfer learning performance is correlated with supervised linear probe performance} The performance of \methodnew 2-spaces is strongly correlated with the average performance of linear probe using CLIP ViT-L/14 and DINOv2 ViT-g/14
($\rho=0.88, p=2.3\times10^{-9}$ for two-sided Pearson correlation coefficient).
}
\label{turtle_linear_datasets_multiple}
\end{minipage}
\end{figure}

\section{Results with Scaling Parameters.}\label{appendix:scaling_parameters}
\renewcommand{\thetable}{\Alph{section}\arabic{table}}
\renewcommand\thefigure{\Alph{section}\arabic{figure}} 
\renewcommand\thealgorithm{\Alph{section}\arabic{algorithm}}
\renewcommand{\theHtable}{\Alph{section}\arabic{table}}
\renewcommand\theHfigure{\Alph{section}\arabic{figure}} 
\renewcommand\theHalgorithm{\Alph{section}\arabic{algorithm}}
\setcounter{table}{0}
\setcounter{figure}{0}
\setcounter{algorithm}{0}

We provide the complete results for \methodnew on each dataset for multiple architectures and sizes scaling in the number of parameters in Table~\ref{tab:big-big-scaling-table}.
Specifically, we consider $8$ CLIP models trained in~\citet{radford2021learning} with scaling model sizes of ResNets (RN50, RN101, RN50x4, RN50x16, RNx64) and Vision Transformers (ViT-B/32, ViT-B/16, ViT-L/14). For each CLIP representation, we train \methodnew 1-space with the aforemention representation spaces, and \methodnew 2-spaces with DINOv2 ViT-g/14 used as the second representation space.

Figure~\ref{fig:scaling_all_resnets} and Figure~\ref{fig:scaling_all_vits} show the average performance of \methodnew trained with different CLIP ResNets and CLIP Vision Transformers, and compare it with the performance of the CLIP zero-shot transfer.
As indicated by the plots, the performance of \methodnew 1-space and \methodnew 2-spaces smoothly improves as the model compute increases.
Moreover, although being fully unsupervised, \methodnew 1-space achieves competitive performance (\textit{i.e.}, with less than $1$ point difference) compared to the zero-shot transfer for RN50, RN101, RN50x4, ViT-B/32, ViT-B/16 and ViT-L/14. 
Using RN50x16 and RN50x64, the performance \methodnew 1-space becomes a little bit worse than zero-shot transfer. However, with the additional DINOv2 representations, \methodnew 2-spaces consistently outperforms zero-shot transfer for all the models by a large margin. For example, \methodnew 2-spaces CLIP ViT-L/14 is on average $4\%$ better than zero-shot transfer. 
For completeness, Figure~\ref{fig:clipRN_26datasets} and Figure~\ref{fig:clipvit_26datasets} also provide the performance of \methodnew and CLIP zero-shot transfer on each individual dataset.

Overall, the performance of \methodnew follows a similar scaling trend as the zero-shot transfer~\cite{radford2021learning}. Furthermore, \methodnew can effectively combine the knowledge of multiple foundation models to further improve the performance of unsupervised transfer.

\begin{figure}[ht]
\vspace{15pt}
\begin{minipage}[t]{0.45\textwidth}
    \centering
    \includegraphics[width = 0.95\linewidth]{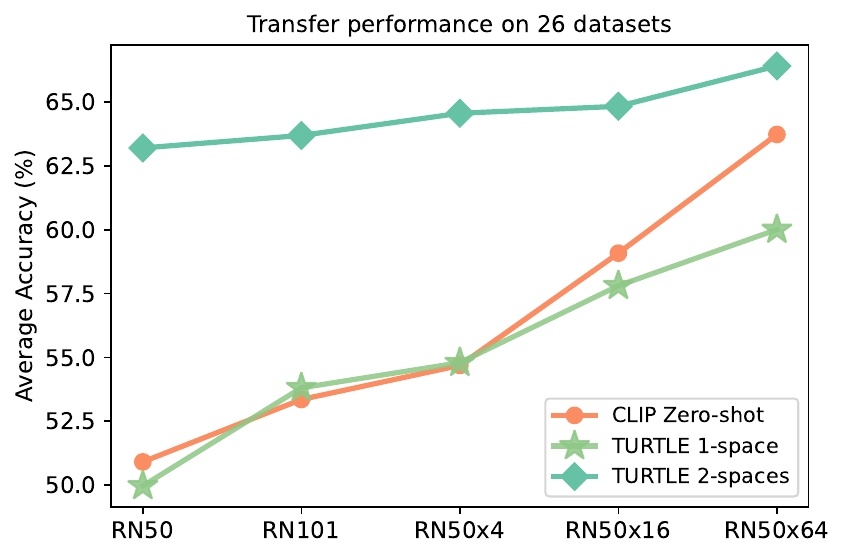}
    \caption{\xhdr{Performance of \methodnew and CLIP zero-shot averaged on 26 vision datasets using ResNets} \methodnew 2-spaces uses DINOv2 ViT-g/14 as the second representation space.}
    \label{fig:scaling_all_resnets}
\end{minipage}
\begin{minipage}[t]{0.05\textwidth}
    ~~~
\end{minipage}
\begin{minipage}[t]{0.44\textwidth}
    \centering
    \includegraphics[width = 0.95\linewidth]{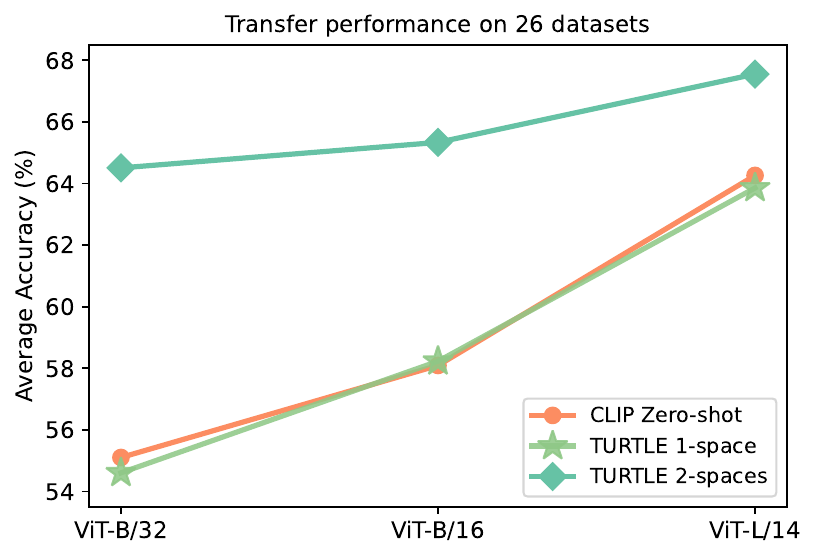}
    \caption{\xhdr{Performance of \methodnew and CLIP zero-shot averaged on 26 vision datasets using Vision Transformers} \methodnew 2-spaces uses DINOv2 ViT-g/14 as the second representation space.}
    \label{fig:scaling_all_vits}
\end{minipage}
\end{figure}

\clearpage 

\begin{figure}[ht]
    \centering
    \includegraphics[width=0.89\linewidth]{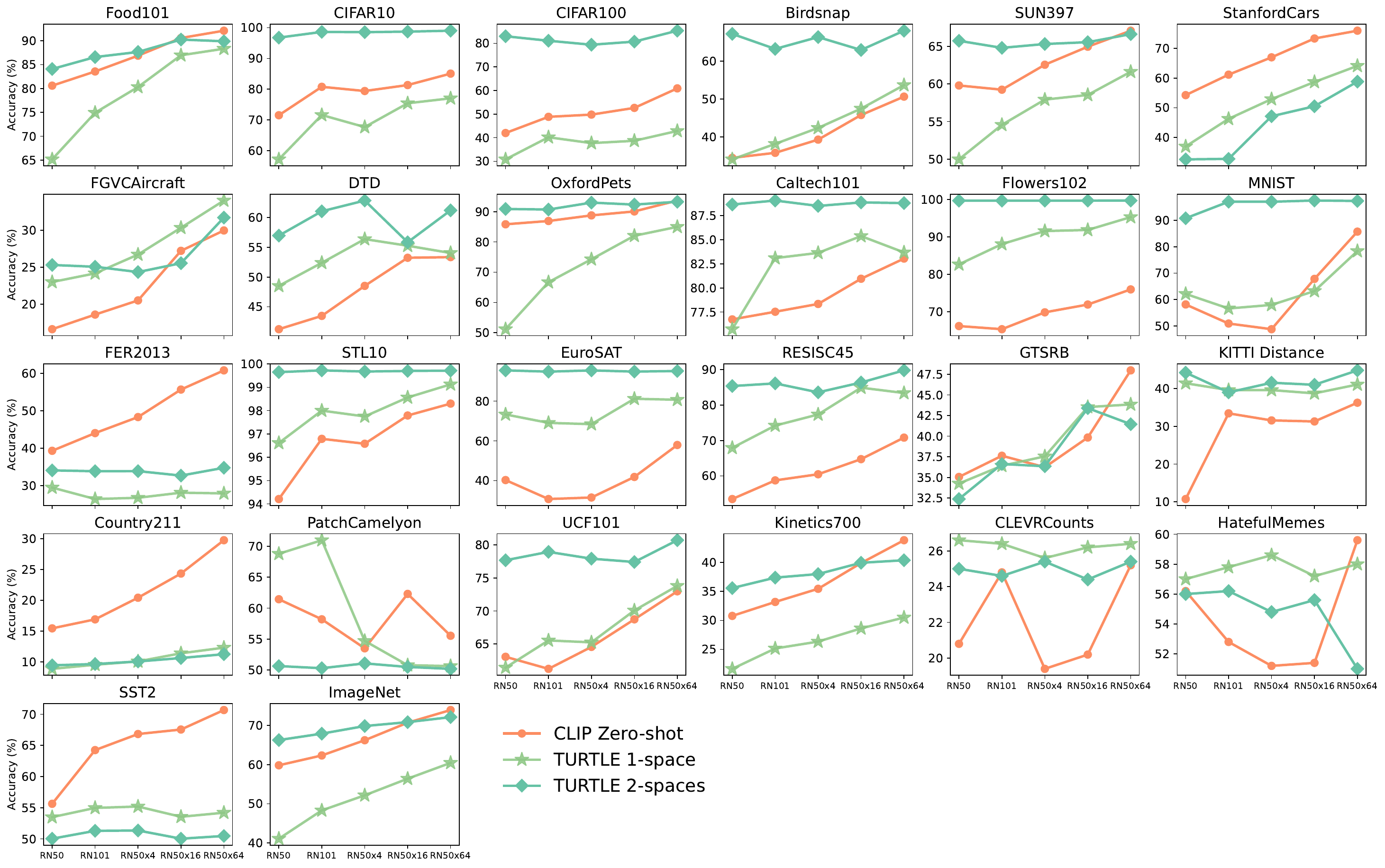}
    \caption{\xhdr{Per dataset performance of \methodnew and CLIP zero-shot using ResNets} \methodnew 2-spaces uses DINOv2 ViT-g/14 as the second representation space.}
    \label{fig:clipRN_26datasets}
\end{figure}

\begin{figure}[ht!]
    \centering
    \includegraphics[width=0.89\linewidth]{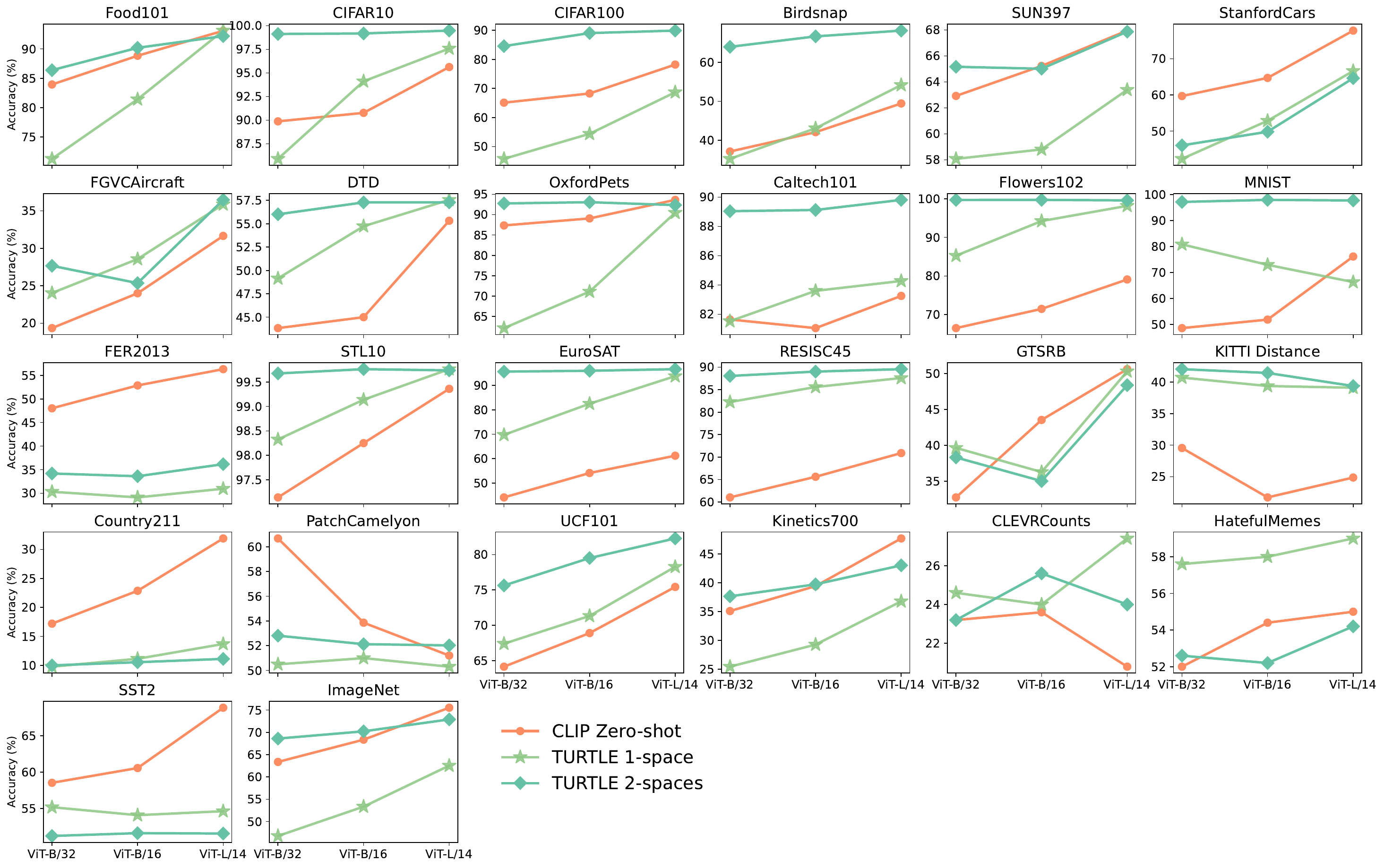}
    \caption{\xhdr{Per dataset performance of \methodnew and CLIP zero-shot using Vision Transformers} \methodnew 2-spaces uses DINOv2 ViT-g/14 as the second representation space.}
    \label{fig:clipvit_26datasets}
\end{figure}

\newpage
\section{Imbalanced Dataset and Entropy Regularization}\label{app:entropyablation}
\renewcommand{\thetable}{\Alph{section}\arabic{table}}
\renewcommand\thefigure{\Alph{section}\arabic{figure}} 
\renewcommand\thealgorithm{\Alph{section}\arabic{algorithm}}
\renewcommand{\theHtable}{\Alph{section}\arabic{table}}
\renewcommand\theHfigure{\Alph{section}\arabic{figure}} 
\renewcommand\theHalgorithm{\Alph{section}\arabic{algorithm}}
\setcounter{table}{0}
\setcounter{figure}{0}
\setcounter{algorithm}{0}
Following~\citet{xu2004maximum, van2020scan, gadetsky2023pursuit}, we use entropy regularization (\ref{turtle_entropy_regularization}) to prevent the task encoder from producing trivial solutions, \textit{i.e.}, assigning all the samples to a single class. By default we set the regularization strength to $\gamma=10$ for all the experiments. Note that the optimal solution of (\ref{turtle_entropy_regularization}) is to produce a labeling with the equal number of samples for each class. However, some of the datasets are not class balanced. In this case, a strong entropy regularization might hurt the learning process. To understand the effect of the entropy regularization, we show the average performance of \methodnew with $\gamma \in \{0, 1, 3, 5, 10\}$ separately on the imbalanced datasets (Birdsnap, FER2013, GTSRB, KITTI, HatefulMemes), and the rest $21$ balanced datasets in Figure~\ref{fig:entropy_reg}. 
The results indicate that the entropy regularization is generally helpful since $\eta=0$ might lead to trivial solutions. Furthermore, for the balanced datasets, \methodnew is robust to the choice of the regularization hyperparameter. While for the imbalanced datasets, a properly chosen regularization parameter could further improve the performance.

\begin{figure}[ht]
    \centering
    \includegraphics[width=0.3\linewidth]{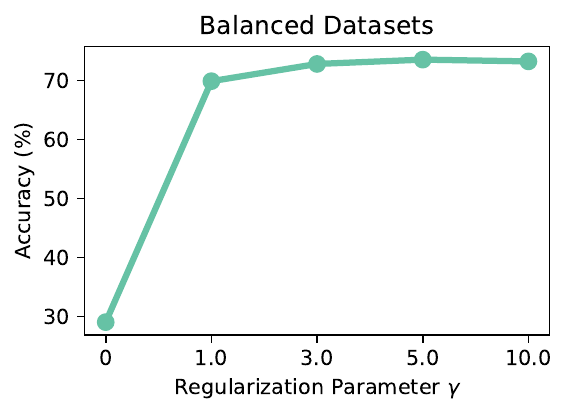}
    \includegraphics[width=0.3\linewidth]{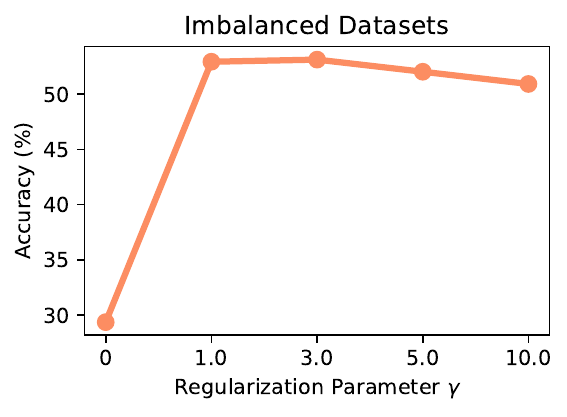}
    \caption{\xhdr{Ablation of the entropy regularization} We show the average performance for class imbalanced datasets (Birdsnap, FER2013, GTSRB, KITTI, HatefulMemes) and class balanced datasets (the rest 21 datasets) for the different entropy regularization strength.}
    \label{fig:entropy_reg}
    \vspace{-8pt}
\end{figure}

\section{\methodnew Trained and Evaluated on Test Split}\label{app:testtestexp}
\renewcommand{\thetable}{\Alph{section}\arabic{table}}
\renewcommand\thefigure{\Alph{section}\arabic{figure}} 
\renewcommand\thealgorithm{\Alph{section}\arabic{algorithm}}
\renewcommand{\theHtable}{\Alph{section}\arabic{table}}
\renewcommand\theHfigure{\Alph{section}\arabic{figure}} 
\renewcommand\theHalgorithm{\Alph{section}\arabic{algorithm}}
\setcounter{table}{0}
\setcounter{figure}{0}
\setcounter{algorithm}{0}
In previous experiments, we train \methodnew on the training split $\mathcal{D}_{tr}$ and evaluate the clustering accuracy on the test split $\mathcal{D}_{te}$. In this section, to study the performance of \methodnew in low data regime, we consider the setting when training and evaluating \methodnew directly on the test split. Figure~\ref{fig:dtrain_dval} compares the performance of \methodnew trained on $\mathcal{D}_{tr}$ and \methodnew trained on $\mathcal{D}_{te}$ on the 26 datasets. Both settings are evaluated on the test split. As shown in the plot, \methodnew trained on $\mathcal{D}_{te}$ achieves nearly identical performance as \methodnew trained on $\mathcal{D}_{tr}$ for $24$ out of $26$ datasets, except Caltech101 and Flowers102.
We found the discrepancy might be attributed to the fact that the Caltech101 and Flowers102 have balanced training split but imbalanced test split. Overall, the results suggest that \methodnew does not require a large amount of data to perform successful unsupervised transfer.

\begin{figure}[ht]
    \centering
    \includegraphics[width=0.41\textwidth]{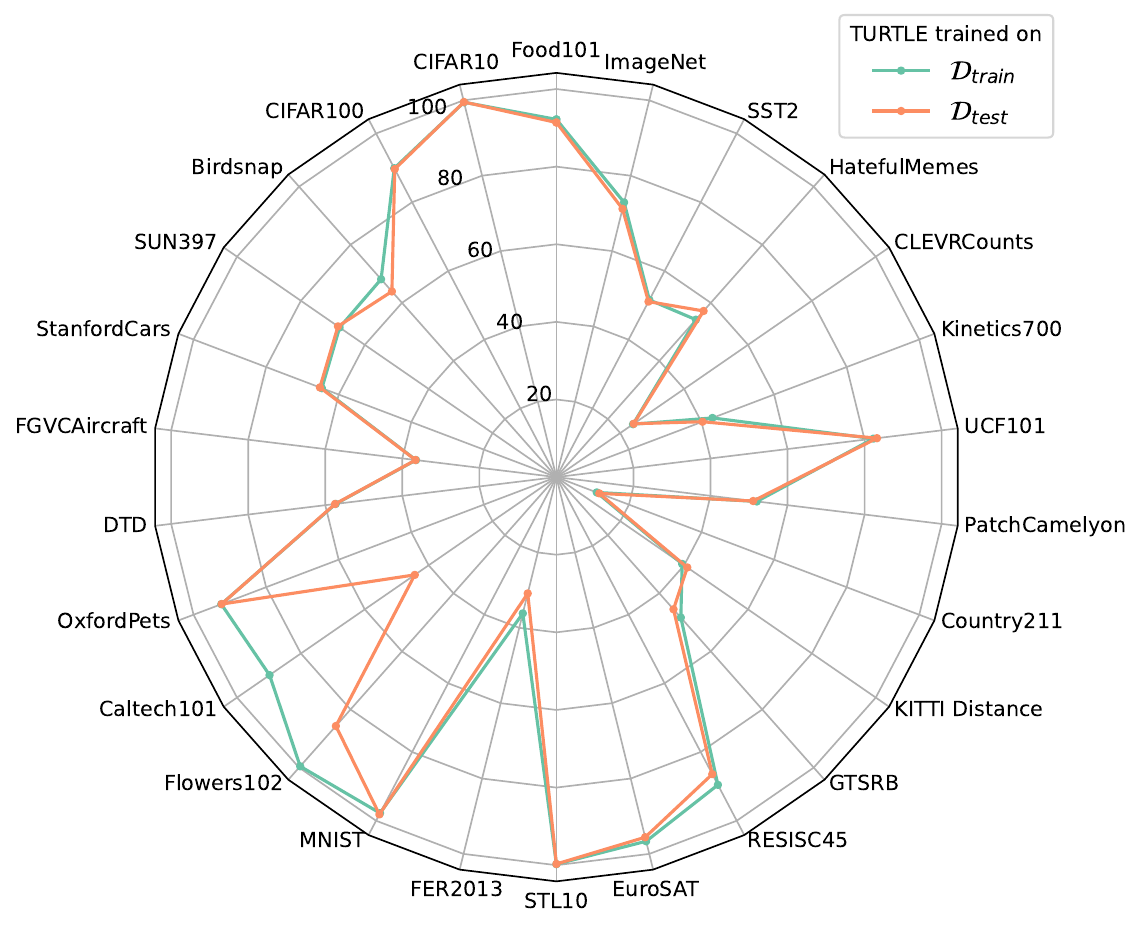}
    \caption{\textbf{\methodnew trained on test split achieves similar performance as \methodnew trained on training split for $24$ out $26$ of datasets}. Results are both evaluated on the test split. The discrepancy of Caltech101 and Flowers102 is because that they are balanced on training split but imbalanced on test split.}
    \label{fig:dtrain_dval}
\end{figure}

\section{Additional Analysis on Fine-grained Classification}\label{app:finegrainedadditional}
\renewcommand{\thetable}{\Alph{section}\arabic{table}}
\renewcommand\thefigure{\Alph{section}\arabic{figure}} 
\renewcommand\thealgorithm{\Alph{section}\arabic{algorithm}}
\renewcommand{\theHtable}{\Alph{section}\arabic{table}}
\renewcommand\theHfigure{\Alph{section}\arabic{figure}} 
\renewcommand\theHalgorithm{\Alph{section}\arabic{algorithm}}
\setcounter{table}{0}
\setcounter{figure}{0}
\setcounter{algorithm}{0}

In previous sections, we have evaluated the performance of \methodnew on 26 datasets, including $6$ fine-grained classification datasets: Food101, Flowers102, Birdsnap, StanfordCars, FGVCAircraft and OxfordPets.
According to the results from Table~\ref{tab:big-big-scaling-table}
and Figure~\ref{fig:turtle_clipdino_zs}, \methodnew outperforms CLIP zero-shot transfer on $3$ datasets (Flowers102, Birdsnap and FGVCAircraft) and performs comparably on $2$ datasets (Food101 and OxfordPets). The results indicate that \methodnew remains effective for the task of fine-grained classification.

To further study the dependence between number of classes and the performance of \methodnew, we perform the additional experiments on 4 datasets from the BREEDS benchmark~\cite{santurkar2020breeds}. These datasets are the subsets of ImageNet that contain both coarse and fine-grained labels. We run \methodnew for each dataset to infer the coarse and fine grained labels separately by specifying the ground truth number of classes for each case. The statistics for each dataset and \methodnew's performance are provided in Table~\ref{tab:breeds_datasets}.

We observe that \methodnew performs worse on the fine-grained classification compared to the coarse-grained classification on \textsc{Living}-17. The result is expected since fine-grained classification is considered to be a more difficult task. However, on \textsc{Entity}-13, \textsc{Entity}-30 and \textsc{Nonloving}-26, the performance of the fine-grained classification is better than the performance of the coarse-grained classification. We hypothesize that this might be due to the high intra-variance of coarse classes, which was also reported in the previous works on unsupervised image classification~\cite{van2020scan}.
In conclusion, the results suggest that the performance of \methodnew is not largely affected by the granularity of the dataset, but rather by the quality of the representations as indicated by Figure \ref{turtle_linear_datasets} and Figure \ref{fig:turtle_imagenet}.

\begin{table*}[ht]
\centering
\caption{\xhdr{BREEDS benchmark and performance of \methodnew} \methodnew column represents the performance of \methodnew on the given dataset with coarse or fine-grained taxonomy.}
\vspace{6pt}
\renewcommand{\arraystretch}{1.1}
\begin{tabular}{l|cccc|cc}
\midrule[1pt]
 \multirow{2}{*}{Dataset} & \multirow{2}{*}{\# Coarse classes} & \multirow{2}{*}{\# Fine classes} & \multirow{2}{*}{Train size} & \multirow{2}{*}{Test size} & \multicolumn{2}{c}{\methodnew} \\
   & & & & & Coarse & Fine \\
 \midrule
 \textsc{Entity}-13 & 13 & 260 & 334,712 & 13,000 & 73.1 & 85.5 \\
 \textsc{Entity}-30 & 30 & 240 & 307,828 & 12,000 & 79.5 & 82.4 \\
 \textsc{Living}-17 & 17 & 68 & 88,400 & 3,400 & 96.0 & 87.0 \\
 \textsc{Nonliving}-26 & 26 & 104 & 132,765 & 5,200 & 76.9 & 78.9 \\
\midrule[1pt]
\end{tabular}
\label{tab:breeds_datasets}
\end{table*}

\end{document}